\pdfoutput=1
\documentclass[11pt]{article}
\usepackage[preprint]{acl}
\usepackage{times}
\usepackage{latexsym}
\usepackage{inconsolata}
\usepackage{microtype}
\usepackage{graphicx}
\usepackage{booktabs}
\usepackage{xspace}
\usepackage{enumitem}
\usepackage{algorithm}
\usepackage[noend]{algpseudocode}
\usepackage{mathtools}
\usepackage{listings}
\usepackage{xcolor}
\usepackage{multirow}
\usepackage{amsmath, amsthm, amssymb}
\usepackage[most]{tcolorbox}
\newtheorem{theorem}{Theorem}[section]
\newtheorem{proposition}[theorem]{Proposition}

\newtheorem{definition}[theorem]{Definition}
\newtheorem{assumption}[theorem]{Assumption}

\newtheorem{lemma}[theorem]{Lemma}
\usepackage{hyperref}
\lstdefinestyle{prompt}{
  basicstyle=\ttfamily\small,
  columns=fullflexible,
  breaklines=true,
  frame=single
}

\title{Batch-of-Thought: Cross-Instance Learning for Enhanced LLM Reasoning}

\author{
  \textbf{Xuan Yang}\textsuperscript{1,2}\thanks{~~Work done during an internship at ByteDance Inc.}, \textbf{Furong Jia}\textsuperscript{1}, \textbf{Roy Xie}\textsuperscript{1}, \textbf{Xi Xiong}\textsuperscript{2}, \\
  \textbf{Hengwei Bian}\textsuperscript{2}, \textbf{Jian Li}\textsuperscript{2}, \textbf{Monica Agrawal}\textsuperscript{1} \\
  \textsuperscript{1}Duke University \qquad \textsuperscript{2}ByteDance Inc. \\
  \texttt{\{xuan.yang, flora.jia, ruoyu.xie, monica.agrawal\}@duke.edu} \\
  \texttt{\{xi.xiong, hengwei.bian, limingjun.tsinghua\}@bytedance.com}
}

\begin{document}
\maketitle

\begin{abstract}
Current Large Language Model reasoning systems process queries independently, discarding valuable cross-instance signals such as shared reasoning patterns and consistency constraints. We introduce Batch-of-Thought (BoT), a training-free method that processes related queries jointly to enable cross-instance learning. By performing comparative analysis across queries, BoT identifies high-quality reasoning templates, detects errors through consistency checks, and amortizes computational costs. We instantiate BoT within a multi-agent reflection architecture (BoT-R), where a Reflector performs joint evaluation to unlock mutual information gain unavailable in isolated processing. Experiments across three model families and six benchmarks demonstrate that BoT-R consistently improves accuracy and confidence calibration while reducing inference costs by up to 61\%. Our theoretical and experimental analysis reveals \emph{when} and \emph{why} batch-aware reasoning benefits LLM systems.\footnote{Code can be found in \url{https://github.com/xuanyang19/BoT}.}
\end{abstract}

\section{Introduction}
Large Language Models (LLMs) \citep{achiam2023gpt, madaan2023self, wei2022chain, shinn2023reflexion, yao2022react} have achieved strong performance across diverse tasks and are increasingly applied in domains such as medical reasoning, question answering, and scientific problem solving \citep{singhal2025toward, mcduff2025towards, nori2025sequential, haas2025simpleqa, wang2023scibench, sun2024scieval}. However, producing reliable answers with well-calibrated confidence remains a challenge~\citep{xiong2024can, ji2023survey}. LLMs often assign high confidence to incorrect answers, which undermines their practical deployment in high-stakes applications where accuracy and reliable uncertainty quantification are essential.
\begin{figure}
    \centering
    \includegraphics[width=\linewidth]{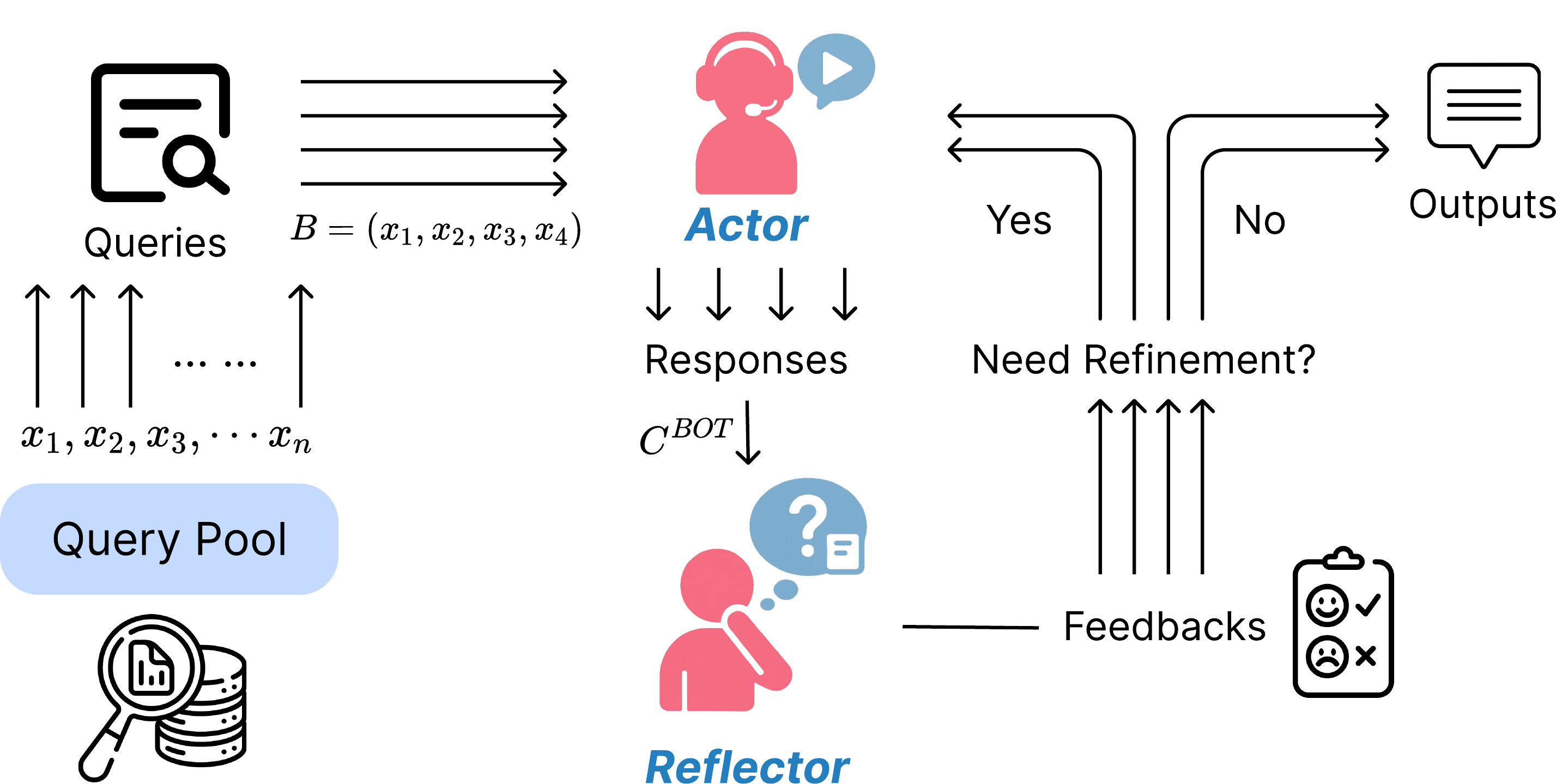}
    \caption{Batch-of-Thought reflection framework. An Actor generates initial responses for a batch of queries. A Reflector then jointly evaluates all responses through comparative analysis, determining whether each should be finalized or refined with feedback.}
    \label{fig:intro}
\end{figure}

Multi-agent LLM systems~\citep{li2023camel, chan2024chateval, guo2024large} extend single-model capabilities through specialized roles and iterative refinement. Despite architectural diversity, existing approaches share a fundamental limitation: they process queries independently. While computationally straightforward, this paradigm discards valuable cross-instance signals. When queries share domain characteristics or structural patterns, isolated processing prevents the system from identifying outliers through comparative assessment, propagating validated knowledge from confident instances to uncertain ones, or detecting errors that emerge only through cross-instance consistency checks.

We introduce \textbf{Batch-of-Thought (BoT)}, a training-free framework that processes related queries jointly to enable cross-instance learning and comparative reasoning. Our key insight is that batch-level reasoning unlocks mutual information gain unavailable in isolated processing. By treating queries as a cohort rather than independent instances, BoT enables comparative analysis, reasoning pattern identification, and distributional uncertainty calibration.

To illustrate this principle, consider fraud detection: evaluating a single seller in isolation provides limited signal. Examining a cohort simultaneously reveals recurring suspicious patterns, distinguishes legitimate domain practices from genuine anomalies, and enables comparative evidence assessment. This principle mirrors James--Stein estimation~\citep{james1961estimation, stein1956inadmissibility, efron1977stein}: pooling information across similar instances improves individual estimates through shrinkage toward the cohort mean (see Appendix~\ref{sec:theory} for theoretical analysis).

We instantiate BoT within a multi-agent reflection architecture~\citep{madaan2023self, shinn2023reflexion}, termed \textbf{BoT-R}, though the principle generalizes to other frameworks including Plan-and-Act~\citep{erdogan2025plan} and Multi-Agent Debate~\citep{du2024improving, liang2024encouraging}. In BoT-R, an Actor generates answer-rationale pairs for a batch of queries, then a Reflector jointly evaluates through comparative analysis---identifying inconsistencies, extracting shared domain knowledge, assessing relative quality, and suggesting refinements. This approach simultaneously improves reasoning quality and computational efficiency by amortizing reflection overhead across the batch. We summarize our main contributions as follows:

\begin{enumerate}
    \item We propose \textbf{Batch-of-Thought (BoT)}, a training-free method that enhances LLM reasoning by processing related queries as cohesive batches, enabling cross-instance learning unavailable in isolated processing.
    \item We instantiate BoT in a reflection-based multi-agent system and conduct experiments across six benchmarks and three model families, demonstrating consistent accuracy improvements and \textbf{46.9\%} average cost reduction.
    \item We theoretically and experimentally analyze how task characteristics and batch composition influence BoT's effectiveness, revealing that interpretive domains benefit substantially from comparative reasoning while symbolic tasks require careful batch design.
\end{enumerate}

\begin{algorithm}[t]
\caption{BoT-R}
\label{alg:bot}
\small
\begin{algorithmic}
\Require Batch $B=\{x_i\}_{i=1}^N$, Actor $\mathcal{A}$, Reflector $\mathcal{R}$, tool set $\mathcal{T}$, max outer rounds $T$, max tool calls $K$
\Ensure Final answers $\{a_i\}_{i=1}^N$ with confidences $\{u_i\}_{i=1}^N$
\State Initialize $c_i \gets \varnothing$, $\forall i\in[N]$; active set $S \gets [N]$
\For{$t = 1$ \textbf{to} $T$}
    \State \textbf{(Parallel)} $(a_i,\rho_i)\gets \mathcal{A}(x_i,\mathcal{T},c_i,K)$ for all $i\in S$
    \State Build reflective context $\mathcal{C} \gets \langle (x_i,a_i,\rho_i)\rangle_{i=1}^N$
    \State \textbf{(Joint)} $(u_i,r_i,c_i)\gets \mathcal{R}(\mathcal{C})$ for all $i\in S$
    \If{$\forall i:\ r_i=0 $}
        \textbf{break}
    \EndIf
    \State $S \gets \{\, i\mid r_i=1 \}$
\EndFor
\State \Return $\{a_i\}, \{u_i\}$
\end{algorithmic}
\end{algorithm}

\begin{table*}[t]
\centering
\small
\resizebox{\textwidth}{!}{%
\begin{tabular}{l|l|c|c|c|c|c|c}
\toprule
\textbf{Model} & \textbf{Method} & \textbf{FraudDet}& \textbf{GPQA}& \textbf{Winogrande}& \textbf{MedQA}& \textbf{PubMedQA}& \textbf{SMS Spam}\\
 & & $n=1793$& $n=448$& $n=1267$& $n=1273$& $n=1000$&$n=1510$\\
\midrule
\multirow{3}{*}{GPT-4o} & ReAct & 0.685 & 0.439 & 0.872 & 0.878 & 0.679 & 0.796\\
 & Reflection & 0.693 & 0.459 & 0.879 & 0.901 & 0.667 & 0.854 \\
 & BoT-R & \textbf{0.740}(4)& \textbf{0.488}(4)& \textbf{0.890}(8)& \textbf{0.904}(8)& \textbf{0.698}(8)& \textbf{0.887}(8)\\
\midrule
\multirow{3}{*}{Llama-3.3-70B} & ReAct & 0.635 & 0.494 & 0.831 & 0.783 & 0.753 & 0.920 \\
 & Reflection & 0.679 & 0.504 & 0.853 & 0.797 & 0.755 & \textbf{0.925} \\
 & BoT-R & \textbf{0.713(4)} & \textbf{0.516(4)} & \textbf{0.862(4)} & \textbf{0.804(4)} & \textbf{0.757(8)} & 0.923(8) \\
\midrule
\multirow{3}{*}{Qwen3-Next-80B} & ReAct & 0.633 & 0.560 & 0.823 & 0.814 & \textbf{0.732} & \textbf{0.946} \\
 & Reflection & 0.639 & 0.636 & 0.869 & 0.846 & 0.681 & 0.900 \\
 & BoT-R & \textbf{0.660(4)} & \textbf{0.657(4)} & \textbf{0.874(4)} & \textbf{0.860(8)} & 0.704(4) & 0.919(8) \\
\bottomrule
\end{tabular}%
}
\caption{Performance comparison of reasoning methods across various base models and datasets (number of queries listed as $n$). Scores represent accuracy. For BoT-R, the number in parentheses denotes the batch size $N \in \{4,8\}$ used to obtain the reported result. The best result in each setting is highlighted in bold. }
\label{tab:main}
\end{table*}

\section{Methods}
Batch-of-Thought (BoT) is a training-free, model-agnostic method that jointly processes batches of queries to improve reasoning quality, confidence calibration, and computational efficiency. We formalize the approach and describe its instantiation within a multi-agent reflection architecture.
\subsection{Problem Formulation}
\label{sec:formulation}

Let $\mathcal{X}$ and $\mathcal{Y}$ denote input and output spaces. Queries arrive in batches $B = \{x_i\}_{i=1}^N \subset \mathcal{X}$. We employ a two-agent architecture with iterative refinement:

\paragraph{Actor $\mathcal{A}$.} A ReAct agent~\citep{yao2022react} that interleaves reasoning traces with tool execution to generate answer-rationale pairs. At iteration $t$, given query $x_i$ and optional critique $c_i^{(t-1)}$ from the previous round:
\begin{equation}
(a_i^{(t)}, \rho_i^{(t)}) = \mathcal{A}(x_i, c_i^{(t-1)}; \text{tools}), \quad a_i^{(t)} \in \mathcal{Y},
\end{equation}
where $a_i^{(t)}$ is the predicted answer, $\rho_i^{(t)}$ is the reasoning trace, and $c_i^{(0)} = \varnothing$ for the initial iteration.

\paragraph{Reflector $\mathcal{R}$.}
A reflection agent~\citep{madaan2023self, shinn2023reflexion} that evaluates the context $\mathcal{C}^{(t)}$ containing all current query-answer-rationale tuples. For each query $i$, it produces:
\begin{equation}
(r_i^{(t)}, u_i^{(t)}, c_i^{(t)}) = \mathcal{R}(\mathcal{C}^{(t)}),
\end{equation}
where $r_i^{(t)} \in \{0,1\}$ indicates whether query $i$ requires refinement, $u_i^{(t)} \in [0,1]$ is a confidence score, and $c_i^{(t)}$ is an actionable critique. If $r_i^{(t)} = 1$, the query proceeds to iteration $t+1$ with critique $c_i^{(t)}$; otherwise, $a_i^{(t)}$ is finalized.

\paragraph{Objective.}
Improve (i) task accuracy, (ii) confidence calibration, and (iii) computational efficiency.

\subsection{Batch-of-Thought (BoT)}
\label{sec:bot}

\paragraph{Formalization.}
Standard per-instance reflection constructs $N$ independent contexts:
\begin{equation}
\mathcal{C}_i^{\text{ind}} = \langle (x_i, a_i, \rho_i) \rangle,
\end{equation}
evaluating each query in isolation. BoT instead constructs a single \emph{shared context}:
\begin{equation}
\label{eq:shared-context}
\mathcal{C}^{\text{BoT}} = \langle (x_1, a_1, \rho_1), \ldots, (x_N, a_N, \rho_N) \rangle,
\end{equation}
and performs joint evaluation:
\begin{equation}
\label{eq:bot-reflection}
\bigl\{(r_i, u_i, c_i)\bigr\}_{i=1}^N = \mathcal{R}\bigl(\mathcal{C}^{\text{BoT}}\bigr).
\end{equation}

\paragraph{Cross-instance mechanisms.}
The shared context enables three synergistic mechanisms: \textbf{(1) Outlier detection:} $\mathcal{R}$ identifies answers that appear plausible in isolation but are inconsistent with peer patterns, propagating high-quality reasoning templates via critiques $\{c_i\}$. \textbf{(2) Distributional calibration:} Confidence scores $u_i$ are calibrated relative to batch statistics $\phi(\mathcal{C}^{\text{BoT}})$ rather than assessed independently, improving uncertainty quantification. \textbf{(3) Computational amortization:} Evaluation rubrics are encoded once per batch, reducing input costs, and joint evaluation enables more accurate refinement decisions, reducing unnecessary Actor-Reflector loops.

The complete BoT-R workflow is detailed in Algorithm~\ref{alg:bot}, which alternates between Actor generation and Reflector evaluation steps until convergence or maximum iterations. 

\paragraph{Theoretical foundation.}
Appendix~\ref{sec:theory} establishes formal guarantees through information-theoretic and statistical analysis, demonstrating that BoT achieves a \textbf{Pareto improvement} over independent processing: simultaneously enhancing accuracy and reducing computational cost.

\section{Experiments}
\label{sec:experiments}
\subsection{Experimental Setup}
We evaluate BoT on six datasets, including five public benchmarks and one newly curated corpus, using both API-based and open-source large language models. Full experimental details are provided in Appendix~\ref{ap:settings}.

\paragraph{Datasets.}
Our evaluation covers diverse reasoning and decision-making tasks: GPQA~\citep{rein2023gpqa}, WinoGrande--debiased~\citep{sakaguchi2021winogrande}, PubMedQA~\citep{jin2019pubmedqa}, MedQA (USMLE)~\citep{jin2021disease}, SMS Spam Detection~\citep{yang2025ad}, and a newly curated fraud-seller detection dataset (Appendix~\ref{ap:fraud}). We additionally use MMLU~\citep{hendrycks2020measuring} for the subject-level analysis in \S\ref{sec:discussion}. Together, these benchmarks span scientific reasoning, commonsense inference, biomedical QA, broad academic knowledge, and real-world anomaly detection.

\paragraph{Metrics.}
We evaluate (i) task accuracy, (ii) token efficiency (input, output, and total token counts), and (iii) confidence calibration using two complementary measures:
(a) the Kolmogorov–Smirnov (KS) statistic~\citep{smirnov1939estimation} between the confidence distributions of correct vs.\ incorrect predictions, and
(b) Expected Calibration Error (ECE; \citealp{guo2017calibration}).

\paragraph{Baselines.}
To isolate BoT's contribution, we compare against two training-free reasoning baselines: \textbf{ReAct}~\citep{yao2022react}, which performs standard single-instance reasoning with optional tool augmentation, and \textbf{Reflection}~\citep{shinn2023reflexion,madaan2023self}, a multi-agent framework employing per-instance self-critique and revision. We then augment the same Actor with BoT's joint, batch-aware reflection to obtain \textbf{BoT-R}. All other factors remain constant across methods.

\paragraph{Models.} We report results from both API and open-source models: GPT-4o-2024-11-20~\citep{hurst2024gpt}, Llama-3.3-70B~\citep{grattafiori2024llama} and Qwen3-Next-80B~\citep{yang2025qwen3}. For the \emph{Fraud Detection} dataset, we enable Brave Search and the Brave Summarizer as tools for retrieval and grounding~\citep{brave2025websearch}.

\begin{table}[t]
\centering
\resizebox{\linewidth}{!}{
\begin{tabular}{l|cc|cc|cc}
\toprule
& \multicolumn{2}{c|}{\textbf{SMS\_Spam}} & \multicolumn{2}{c|}{\textbf{GPQA}}  & \multicolumn{2}{c}{\textbf{Winogrande}}\\
\textbf{Method}& Cost & $\Delta\%$ & Cost & $\Delta\%$ & Cost & $\Delta\%$ \\
\midrule
\multicolumn{7}{c}{\textit{Actor}}\\
\midrule
Reflection & \$3.61 & -- & \$4.70 & -- & \$2.57 & -- \\
BoT (4)     &\$2.41 & 33.25\% & \$3.68 & 21.71\% & \$2.27 & 11.75\% \\
BoT (8)     & \$2.06 & 42.96\% & \$3.42 & 27.24\% & \$1.99 & 22.54\% \\
\midrule
\multicolumn{7}{c}{\textit{Reflector}}\\
\midrule
Reflection & \$6.39 & -- & \$2.76 & -- & \$3.40 & -- \\
BoT (4)     & \$2.44 & 61.85\% & \$1.49 & 45.99\% & \$1.97 & 42.10\% \\
BoT (8)     & \$1.84 & 71.12\% & \$1.17 & 57.65\% & \$1.52 & 55.31\% \\
\midrule
\multicolumn{7}{c}{\textit{Total}}\\
\midrule
Reflection & \$10.00& -- & \$7.46 & -- & \$5.97 & -- \\
BoT (4)     & \$4.85& 51.52\%& \$5.17 & 30.68\% & \$4.24 & 29.04\% \\
BoT (8)     & \$3.90& 60.95\%& \$4.59 & 38.48\% & \$3.51 & 41.21\% \\
\bottomrule
\end{tabular}
}
\caption{Total token cost and relative reduction ($\Delta\%$) for each dataset (GPT-4o). $\Delta\%$ is computed for BoT(4) and BoT(8) relative to Reflection cost.}
\label{tab:token_cost_small}
\end{table}

\begin{table}[t]
\centering
\small
\resizebox{\linewidth}{!}{
\begin{tabular}{l|cc|cc|cc}
\toprule
\multirow{2}{*}{\textbf{Method}}
& \multicolumn{2}{c}{\textbf{SMS\_Spam}}
& \multicolumn{2}{c}{\textbf{GPQA}}
& \multicolumn{2}{c}{\textbf{Winogrande}} \\
\cmidrule(lr){2-3}\cmidrule(lr){4-5}\cmidrule(lr){6-7}
& KS↑ & ECE↓ & KS↑  & ECE↓ & KS↑  & ECE↓ \\
\midrule
ReAct       &  0.256  &  0.176  &  0.181  &  0.372  &  0.273  &  0.113  \\
Reflection  &  0.360  &  0.104  &  0.265  &  0.329  &  0.376  &  0.035  \\
BoT-R       &  \textbf{0.633}  &  \textbf{0.063}  &  \textbf{0.368}  &  \textbf{0.317}  &  \textbf{0.442}  &  \textbf{0.013}  \\
\bottomrule
\end{tabular}
}
\caption{Confidence calibration across datasets (GPT-4o). Each entry reports the KS statistic between the confidence distributions of correct vs.\ incorrect answers and the ECE score. Higher KS ($\uparrow$) is better, while lower ECE ($\downarrow$) is better.}
\label{tab:conf}
\end{table}

\begin{table}[t]
\centering
\small
\resizebox{\linewidth}{!}{
\begin{tabular}{l|ccccc}
\toprule
Method & med\&bio & hum & social & math & sci \\
\midrule
ReAct       & 0.887 & 0.805 & 0.915 & 0.763 & 0.797 \\
Reflection  & 0.886 & 0.825 & 0.915 & \textbf{0.865} & \textbf{0.843} \\
BoT-R       & \textbf{0.891} & \textbf{0.837} & \textbf{0.922} & 0.853 & 0.832 \\
\bottomrule
\end{tabular}
}
\caption{MMLU dataset subject-wise accuracy (GPT-4o). The highest score for each subject is in bold.}
\label{tab:method_by_subject}
\end{table}

\subsection{Main Results}
\label{sec:main_results}

\paragraph{Accuracy.}
As shown in Table~\ref{tab:main}, we compare task performance after integrating BoT into the reflection framework. From this ablation study, BoT-R is consistently competitive and typically the strongest overall variant, across datasets and base models, demonstrating the effectiveness of BoT. Gains are most visible on higher-variance, decision-heavy tasks where per-instance reflection can be brittle. For example, under GPT-4o, \textsc{BoT-R} improves FraudDet and GPQA by $+4.7$ and $+2.9$ accuracy points over Reflection, respectively, and yields an average improvement of $+2.6$ points across all six datasets. On near-saturated benchmarks (e.g., SMS Spam with Qwen3), where base performance is already high, the accuracy headroom is limited and BoT-R may not always outperform simpler baselines, suggesting that BoT is most beneficial when cross-instance comparison provides additional corrective signal.

\paragraph{Efficiency.}
We further evaluate efficiency using a unified reference pricing scheme based on production-grade GPT-4o pricing. Table~\ref{tab:token_cost_small} shows that BoT-R substantially reduces overall token cost, achieving \textbf{46.9\%} average savings across the three representative benchmarks at batch size $8$, and up to \textbf{61\%} reduction on \texttt{SMS\_Spam}. Batch formation is not a hidden cost: sequential batching is a single pass over the input stream, and semantic clustering adds only 0.1--0.6\% of the token spend (Appendix~\ref{ap:formation}). BoT-R also dominates majority voting baselines ($s=3$), which is less accurate at $1.3$--$6.7\times$ the cost (Appendix~\ref{ap:majority}). These reductions indicate that batch-aware reflection effectively amortizes reflective reasoning across instances while improving task performance.

\paragraph{Calibration.}
Table~\ref{tab:conf} shows that BoT-R improves confidence reliability under GPT-4o. It increases KS and reduces ECE across all three datasets (e.g., on \texttt{SMS\_Spam}, KS $0.360 \rightarrow 0.633$ and ECE $0.104 \rightarrow 0.063$). The pattern matches the collective-signal perspective of Lemma~\ref{lem:neff}: batch statistics carry usable signal when the effective sample size $N_{\mathrm{eff}}$ is meaningfully above $1$, as expected for moderate error correlation at $N \in \{4,8\}$, and Proposition~\ref{prop:dominance} then bounds the calibration gain attainable by conditioning on them.

Overall, the results are consistent with our theoretical analysis: batch-aware reflection yields a favorable accuracy--cost--calibration profile, with robust accuracy gains and calibration improvements on diverse tasks and consistent efficiency.
Appendix~\ref{ap:case} complements these aggregate results with two qualitative traces: one illustrates cross-instance contradiction detection, and the other illustrates comparative quality assessment followed by selective re-evaluation.
Appendix~\ref{ap:token} provides the full token breakdown, and Appendix~\ref{ap:extra_metrics} reports alternative calibration metrics, paired significance tests, and run-to-run variance.

\begin{table}[t]
\centering
\footnotesize
\setlength{\tabcolsep}{5pt}
\begin{tabular}{lccc}
\toprule
\textbf{Dataset} & \textbf{No-batch} & \textbf{Sequential} & \textbf{Semantic} \\
\midrule
FraudDet   & 0.693 & 0.740 & \textbf{0.768} \\
MedQA      & 0.901 & \textbf{0.904} & 0.902 \\
PubMedQA   & 0.667 & \textbf{0.698} & 0.697 \\
Winogrande & 0.879 & 0.890 & \textbf{0.897} \\
SMS Spam   & 0.854 & 0.887 & \textbf{0.902} \\
GPQA       & 0.459 & \textbf{0.488} & 0.486 \\
\bottomrule
\end{tabular}
\caption{Accuracy across batching strategies (GPT-4o).  The highest score for each dataset is in bold. }
\label{tab:semantic_batching_full}
\end{table}

\section{Discussion}
\label{sec:discussion}
\paragraph{RQ1: Which domains benefit most from BoT?}
Table~\ref{tab:method_by_subject} shows that BoT-R yields the largest gains on interpretive and judgment-driven domains, including humanities, social sciences, and medicine \& biology.  These tasks admit multiple plausible reasoning paths and partial cues, making comparative evaluation across instances especially informative. In contrast, domains dominated by exact symbolic derivation, such as mathematics and parts of the physical sciences, exhibit marginal or slightly negative changes. This suggests that batch-level consensus is less effective when correctness depends on exact derivation rather than comparative plausibility (More details in Appendix~\ref{sec:mmlu}). This pattern aligns with the coherence--diversity trade-off formalized in Appendix~\ref{sec:theory:neff}: interpretive tasks tend to satisfy moderate similarity and moderate error correlation, yielding a meaningful effective sample size and enabling collective signal amplification. By contrast, symbolic domains often sit at the wrong end of this trade-off: small derivation errors are highly correlated within a batch, limiting the benefit of aggregation and occasionally amplifying shared mistakes.

\paragraph{RQ2: How does batching strategy influence performance?}
We compare three strategies (Table~\ref{tab:semantic_batching_full}): \textbf{no-batch} (per-instance reflection), \textbf{sequential} batching (queries in natural arrival order; our default), and \textbf{semantic} batching (K-means over E5-Mistral-7B embeddings~\citep{wang2024improving}). Two findings stand out. First, \emph{how} queries are batched matters less than \emph{that} they are batched: sequential grouping already improves over no-batch on all six datasets ($+3.87\%$ average relative gain). Many benchmarks carry latent topical coherence that BoT-R exploits without explicit clustering. Second, semantic coherence adds the most on heterogeneous datasets: \texttt{FraudDet} improves from 0.740 to 0.768 and \texttt{SMS\_Spam} from 0.887 to 0.902 over sequential batching ($+4.83\%$ average relative gain over no-batch). For datasets drawn from a narrow latent distribution (\texttt{MedQA}, \texttt{PubMedQA}), the two strategies are nearly indistinguishable: naive batching already satisfies the coherence conditions (Appendix~\ref{sec:semantic_batching}).

\paragraph{RQ3: How does batch size interact with task structure?}
While larger batches theoretically provide richer comparative context, a sweep over $N \in \{1,\dots,16\}$ (Table~\ref{tab:batch_extended}, Appendix~\ref{ap:batch_size}) reveals a non-monotonic relationship between batch size and performance. Moderate batch sizes ($N \in \{4,8\}$) are a strong default, though the best setting is domain-dependent. Larger batches are practically constrained by (i) context window saturation, which forces rationale compression and degrades fine-grained reasoning, and (ii) increased heterogeneity, which dilutes informative cross-instance comparisons. As a result, performance often peaks at intermediate batch sizes despite stronger theoretical aggregation benefits.

\paragraph{Implications and outlook.}
Overall, these findings suggest that \textsc{BoT} is most effective when batches exhibit sufficient, but not excessive, coherence, and when tasks benefit from comparative judgment rather than exact symbolic derivation. Importantly, the robustness of sequential batching indicates that \textsc{BoT-R} remains practical in streaming or latency-sensitive settings, where semantic clustering may be infeasible. Future work may explore adaptive cohorting strategies that dynamically balance coherence, batch size, and latency, as well as extensions to domains requiring symbolic guarantees.

\section{Conclusion}
We introduced \textbf{Batch-of-Thought (BoT)}, a training-free, model-agnostic approach that processes related queries as a batch so an agent can compare across instances, share knowledge, and produce better-calibrated confidence while amortizing computation. BoT yields higher accuracy, lower token cost, and improved calibration across settings, including on our newly curated Seller Fraud Detection benchmark, and characterizes \emph{when} batching helps and \emph{when} it degrades gracefully.

\section*{Limitations}
The efficacy of Batch-of-Thought (BoT) is subject to several constraints. Its gains depend on batch composition, and poorly matched queries can cause negative transfer. Joint prompting also increases context length, limiting applicability to long-input tasks. While we instantiate BoT in an Actor–Reflector architecture, the idea generalizes naturally to other multi-agent designs (e.g., Plan-and-Act, Debate), empirical validation of such integrations remains open. Finally, our experiments batch items within a single public benchmark and evaluation session, so they do not create cross-user privacy risks. In multi-user deployments, however, a shared prompt could expose one user's content to another; batches should therefore remain within the trusted task boundary.

\section*{Ethics Statement}
In writing this paper, we used an AI assistant to correct grammatical errors. During the coding process, we utilized AI tools for code completion. All datasets and models used in our experiments are publicly accessible. Our newly introduced Seller Fraud Detection benchmark contains only publicly available information and does not include any private or sensitive data. The Seller Fraud Detection benchmark was developed with human expert annotation. All annotators were compensated fairly for their time and expertise at rates exceeding standard professional compensation in their region. Annotators were provided with clear guidelines and had the option to decline participation at any time.

\bibliography{bunch_1}

\begin{thebibliography}{59}
\providecommand{\natexlab}[1]{#1}

\bibitem[{Achiam et~al.(2023)Achiam, Adler, Agarwal, Ahmad, Akkaya, Aleman,
  Almeida, Altenschmidt, Altman, Anadkat et~al.}]{achiam2023gpt}
Josh Achiam, Steven Adler, Sandhini Agarwal, Lama Ahmad, Ilge Akkaya,
  Florencia~Leoni Aleman, Diogo Almeida, Janko Altenschmidt, Sam Altman,
  Shyamal Anadkat, and 1 others. 2023.
\newblock Gpt-4 technical report.
\newblock \emph{arXiv preprint arXiv:2303.08774}.

\bibitem[{{Anthropic}(2024)}]{anthropic2024claude}
{Anthropic}. 2024.
\newblock \href
  {https://www-cdn.anthropic.com/de8ba9b01c9ab7cbabf5c33b80b7bbc618857627/Model_Card_Claude_3.pdf}
  {The {Claude} 3 model family: {Opus}, {Sonnet}, {Haiku}}.
\newblock Technical report, Anthropic.
\newblock Accessed: 2025-10-07.

\bibitem[{{Brave Software, Inc.}(2025)}]{brave2025websearch}
{Brave Software, Inc.} 2025.
\newblock Brave web search {API} documentation.
\newblock
  \url{https://api-dashboard.search.brave.com/app/documentation/web-search/get-started}.
\newblock Accessed: 2025-10-07.

\bibitem[{Brown et~al.(2020)Brown, Mann, Ryder, Subbiah, Kaplan, Dhariwal,
  Neelakantan, Shyam, Sastry, Askell et~al.}]{brown2020language}
Tom Brown, Benjamin Mann, Nick Ryder, Melanie Subbiah, Jared~D Kaplan, Prafulla
  Dhariwal, Arvind Neelakantan, Pranav Shyam, Girish Sastry, Amanda Askell, and
  1 others. 2020.
\newblock Language models are few-shot learners.
\newblock \emph{Advances in neural information processing systems},
  33:1877--1901.

\bibitem[{Chan et~al.(2024)Chan, Chen, Su, Yu, Xue, Zhang, Fu, and
  Liu}]{chan2024chateval}
Chi-Min Chan, Weize Chen, Yusheng Su, Jianxuan Yu, Wei Xue, Shanghang Zhang,
  Jie Fu, and Zhiyuan Liu. 2024.
\newblock Chateval: Towards better llm-based evaluators through multi-agent
  debate.
\newblock In \emph{International conference on learning representations},
  volume 2024, pages 9079--9093.

\bibitem[{Chen et~al.(2024)Chen, Saha, and Bansal}]{chen2024reconcile}
Justin Chen, Swarnadeep Saha, and Mohit Bansal. 2024.
\newblock Reconcile: Round-table conference improves reasoning via consensus
  among diverse llms.
\newblock In \emph{Proceedings of the 62nd Annual Meeting of the Association
  for Computational Linguistics (Volume 1: Long Papers)}, pages 7066--7085.

\bibitem[{Cheng et~al.(2023)Cheng, Kasai, and Yu}]{cheng2023batch}
Zhoujun Cheng, Jungo Kasai, and Tao Yu. 2023.
\newblock Batch prompting: Efficient inference with large language model apis.
\newblock In \emph{Proceedings of the 2023 Conference on Empirical Methods in
  Natural Language Processing: Industry Track}, pages 792--810.

\bibitem[{Cover and Thomas(2006)}]{cover2006elements}
Thomas~M. Cover and Joy~A. Thomas. 2006.
\newblock \emph{Elements of Information Theory, 2/e}.
\newblock John Wiley and Sons.

\bibitem[{Du et~al.(2024)Du, Li, Torralba, Tenenbaum, and
  Mordatch}]{du2024improving}
Yilun Du, Shuang Li, Antonio Torralba, Joshua~B Tenenbaum, and Igor Mordatch.
  2024.
\newblock Improving factuality and reasoning in language models through
  multiagent debate.
\newblock In \emph{Forty-first international conference on machine learning}.

\bibitem[{Efron and Morris(1977)}]{efron1977stein}
Bradley Efron and Carl Morris. 1977.
\newblock Stein's paradox in statistics.
\newblock \emph{Scientific American}, 236(5):119--127.

\bibitem[{Erdogan et~al.(2025)Erdogan, Lee, Kim, Moon, Furuta, Anumanchipalli,
  Keutzer, and Gholami}]{erdogan2025plan}
Lutfi~Eren Erdogan, Nicholas Lee, Sehoon Kim, Suhong Moon, Hiroki Furuta,
  Gopala Anumanchipalli, Kurt Keutzer, and Amir Gholami. 2025.
\newblock Plan-and-act: Improving planning of agents for long-horizon tasks.
\newblock \emph{arXiv preprint arXiv:2503.09572}.

\bibitem[{Finn et~al.(2017)Finn, Abbeel, and Levine}]{finn2017model}
Chelsea Finn, Pieter Abbeel, and Sergey Levine. 2017.
\newblock Model-agnostic meta-learning for fast adaptation of deep networks.
\newblock In \emph{International conference on machine learning}, pages
  1126--1135. PMLR.

\bibitem[{Gneiting and Raftery(2007)}]{gneiting2007strictly}
Tilmann Gneiting and Adrian~E Raftery. 2007.
\newblock Strictly proper scoring rules, prediction, and estimation.
\newblock \emph{Journal of the American statistical Association},
  102(477):359--378.

\bibitem[{Grattafiori et~al.(2024)Grattafiori, Dubey, Jauhri, Pandey, Kadian,
  Al-Dahle, Letman, Mathur, Schelten, Vaughan et~al.}]{grattafiori2024llama}
Aaron Grattafiori, Abhimanyu Dubey, Abhinav Jauhri, Abhinav Pandey, Abhishek
  Kadian, Ahmad Al-Dahle, Aiesha Letman, Akhil Mathur, Alan Schelten, Alex
  Vaughan, and 1 others. 2024.
\newblock The llama 3 herd of models.
\newblock \emph{arXiv preprint arXiv:2407.21783}.

\bibitem[{Guo et~al.(2017)Guo, Pleiss, Sun, and
  Weinberger}]{guo2017calibration}
Chuan Guo, Geoff Pleiss, Yu~Sun, and Kilian~Q Weinberger. 2017.
\newblock On calibration of modern neural networks.
\newblock In \emph{International conference on machine learning}, pages
  1321--1330. PMLR.

\bibitem[{Guo et~al.(2024)Guo, Chen, Wang, Chang, Pei, Chawla, Wiest, and
  Zhang}]{guo2024large}
Taicheng Guo, Xiuying Chen, Yaqi Wang, Ruidi Chang, Shichao Pei, Nitesh~V
  Chawla, Olaf Wiest, and Xiangliang Zhang. 2024.
\newblock Large language model based multi-agents: A survey of progress and
  challenges.
\newblock \emph{arXiv preprint arXiv:2402.01680}.

\bibitem[{Haas et~al.(2025)Haas, Yona, D'Antonio, Goldshtein, and
  Das}]{haas2025simpleqa}
Lukas Haas, Gal Yona, Giovanni D'Antonio, Sasha Goldshtein, and Dipanjan Das.
  2025.
\newblock Simpleqa verified: A reliable factuality benchmark to measure
  parametric knowledge.
\newblock \emph{arXiv preprint arXiv:2509.07968}.

\bibitem[{Hendrycks et~al.(2020)Hendrycks, Burns, Basart, Zou, Mazeika, Song,
  and Steinhardt}]{hendrycks2020measuring}
Dan Hendrycks, Collin Burns, Steven Basart, Andy Zou, Mantas Mazeika, Dawn
  Song, and Jacob Steinhardt. 2020.
\newblock Measuring massive multitask language understanding.
\newblock \emph{arXiv preprint arXiv:2009.03300}.

\bibitem[{Hong et~al.(2024)Hong, Zhuge, Chen, Zheng, Cheng, Wang, Zhang, Yau,
  Lin, Zhou et~al.}]{hong2024metagpt}
Sirui Hong, Mingchen Zhuge, Jonathan Chen, Xiawu Zheng, Yuheng Cheng, Jinlin
  Wang, Ceyao Zhang, Steven Yau, Zijuan Lin, Liyang Zhou, and 1 others. 2024.
\newblock Metagpt: Meta programming for a multi-agent collaborative framework.
\newblock In \emph{International Conference on Learning Representations},
  volume 2024, pages 23247--23275.

\bibitem[{Hurst et~al.(2024)Hurst, Lerer, Goucher, Perelman, Ramesh, Clark,
  Ostrow, Welihinda, Hayes, Radford et~al.}]{hurst2024gpt}
Aaron Hurst, Adam Lerer, Adam~P Goucher, Adam Perelman, Aditya Ramesh, Aidan
  Clark, AJ~Ostrow, Akila Welihinda, Alan Hayes, Alec Radford, and 1 others.
  2024.
\newblock Gpt-4o system card.
\newblock \emph{arXiv preprint arXiv:2410.21276}.

\bibitem[{Ioffe and Szegedy(2015)}]{ioffe2015batch}
Sergey Ioffe and Christian Szegedy. 2015.
\newblock Batch normalization: Accelerating deep network training by reducing
  internal covariate shift.
\newblock In \emph{International conference on machine learning}, pages
  448--456. pmlr.

\bibitem[{James and Stein(1961)}]{james1961estimation}
William James and Charles Stein. 1961.
\newblock Estimation with quadratic loss.
\newblock In \emph{Proceedings of the fourth Berkeley symposium on mathematical
  statistics and probability}, volume~1, pages 361--379. University of
  California Press.

\bibitem[{Ji et~al.(2023)Ji, Lee, Frieske, Yu, Su, Xu, Ishii, Bang, Madotto,
  and Fung}]{ji2023survey}
Ziwei Ji, Nayeon Lee, Rita Frieske, Tiezheng Yu, Dan Su, Yan Xu, Etsuko Ishii,
  Ye~Jin Bang, Andrea Madotto, and Pascale Fung. 2023.
\newblock Survey of hallucination in natural language generation.
\newblock \emph{ACM computing surveys}, 55(12):1--38.

\bibitem[{Jin et~al.(2021)Jin, Pan, Oufattole, Weng, Fang, and
  Szolovits}]{jin2021disease}
Di~Jin, Eileen Pan, Nassim Oufattole, Wei-Hung Weng, Hanyi Fang, and Peter
  Szolovits. 2021.
\newblock What disease does this patient have? a large-scale open domain
  question answering dataset from medical exams.
\newblock \emph{Applied Sciences}, 11(14):6421.

\bibitem[{Jin et~al.(2019)Jin, Dhingra, Liu, Cohen, and Lu}]{jin2019pubmedqa}
Qiao Jin, Bhuwan Dhingra, Zhengping Liu, William Cohen, and Xinghua Lu. 2019.
\newblock Pubmedqa: A dataset for biomedical research question answering.
\newblock In \emph{Proceedings of the 2019 conference on empirical methods in
  natural language processing and the 9th international joint conference on
  natural language processing (EMNLP-IJCNLP)}, pages 2567--2577.

\bibitem[{Kadavath et~al.(2022)Kadavath, Conerly, Askell, Henighan, Drain,
  Perez, Schiefer, Hatfield-Dodds, DasSarma, Tran-Johnson
  et~al.}]{kadavath2022language}
Saurav Kadavath, Tom Conerly, Amanda Askell, Tom Henighan, Dawn Drain, Ethan
  Perez, Nicholas Schiefer, Zac Hatfield-Dodds, Nova DasSarma, Eli
  Tran-Johnson, and 1 others. 2022.
\newblock Language models (mostly) know what they know.
\newblock \emph{arXiv preprint arXiv:2207.05221}.

\bibitem[{Kuhn et~al.(2023)Kuhn, Gal, and Farquhar}]{kuhn2023semantic}
Lorenz Kuhn, Yarin Gal, and Sebastian Farquhar. 2023.
\newblock Semantic uncertainty: Linguistic invariances for uncertainty
  estimation in natural language generation.
\newblock \emph{arXiv preprint arXiv:2302.09664}.

\bibitem[{Lee et~al.(2019)Lee, Lee, Kim, Kosiorek, Choi, and Teh}]{lee2019set}
Juho Lee, Yoonho Lee, Jungtaek Kim, Adam Kosiorek, Seungjin Choi, and Yee~Whye
  Teh. 2019.
\newblock Set transformer: A framework for attention-based
  permutation-invariant neural networks.
\newblock In \emph{International conference on machine learning}, pages
  3744--3753. PMLR.

\bibitem[{Li et~al.(2023)Li, Hammoud, Itani, Khizbullin, and
  Ghanem}]{li2023camel}
Guohao Li, Hasan Hammoud, Hani Itani, Dmitrii Khizbullin, and Bernard Ghanem.
  2023.
\newblock Camel: Communicative agents for ''mind'' exploration of large
  language model society.
\newblock \emph{Advances in neural information processing systems},
  36:51991--52008.

\bibitem[{Liang et~al.(2024)Liang, He, Jiao, Wang, Wang, Wang, Yang, Shi, and
  Tu}]{liang2024encouraging}
Tian Liang, Zhiwei He, Wenxiang Jiao, Xing Wang, Yan Wang, Rui Wang, Yujiu
  Yang, Shuming Shi, and Zhaopeng Tu. 2024.
\newblock Encouraging divergent thinking in large language models through
  multi-agent debate.
\newblock In \emph{Proceedings of the 2024 conference on empirical methods in
  natural language processing}, pages 17889--17904.

\bibitem[{Lin et~al.(2022)Lin, Hilton, and Evans}]{lin2022teaching}
Stephanie Lin, Jacob Hilton, and Owain Evans. 2022.
\newblock Teaching models to express their uncertainty in words.
\newblock \emph{arXiv preprint arXiv:2205.14334}.

\bibitem[{Lohr(2023)}]{lohr2023sampling}
Sharon~L Lohr. 2023.
\newblock \emph{Sampling: design and analysis}, volume~1.
\newblock Duxbury press Pacific Grove, California.

\bibitem[{Madaan et~al.(2023)Madaan, Tandon, Gupta, Hallinan, Gao, Wiegreffe,
  Alon, Dziri, Prabhumoye, Yang et~al.}]{madaan2023self}
Aman Madaan, Niket Tandon, Prakhar Gupta, Skyler Hallinan, Luyu Gao, Sarah
  Wiegreffe, Uri Alon, Nouha Dziri, Shrimai Prabhumoye, Yiming Yang, and 1
  others. 2023.
\newblock Self-refine: Iterative refinement with self-feedback.
\newblock \emph{Advances in neural information processing systems},
  36:46534--46594.

\bibitem[{McDuff et~al.(2025)McDuff, Schaekermann, Tu, Palepu, Wang, Garrison,
  Singhal, Sharma, Azizi, Kulkarni et~al.}]{mcduff2025towards}
Daniel McDuff, Mike Schaekermann, Tao Tu, Anil Palepu, Amy Wang, Jake Garrison,
  Karan Singhal, Yash Sharma, Shekoofeh Azizi, Kavita Kulkarni, and 1 others.
  2025.
\newblock Towards accurate differential diagnosis with large language models.
\newblock \emph{Nature}, 642(8067):451--457.

\bibitem[{Mitchell et~al.(2022)Mitchell, Noh, Li, Armstrong, Agarwal, Liu,
  Finn, and Manning}]{mitchell2022enhancing}
Eric Mitchell, Joseph Noh, Siyan Li, Will Armstrong, Ananth Agarwal, Patrick
  Liu, Chelsea Finn, and Christopher~D Manning. 2022.
\newblock Enhancing self-consistency and performance of pre-trained language
  models through natural language inference.
\newblock In \emph{Proceedings of the 2022 Conference on Empirical Methods in
  Natural Language Processing}, pages 1754--1768.

\bibitem[{Murphy(1973)}]{murphy1973new}
Allan~H Murphy. 1973.
\newblock A new vector partition of the probability score.
\newblock \emph{Journal of Applied Meteorology and Climatology},
  12(4):595--600.

\bibitem[{Nori et~al.(2025)Nori, Daswani, Kelly, Lundberg, Ribeiro, Wilson,
  Liu, Sounderajah, Carlson, Lungren et~al.}]{nori2025sequential}
Harsha Nori, Mayank Daswani, Christopher Kelly, Scott Lundberg, Marco~Tulio
  Ribeiro, Marc Wilson, Xiaoxuan Liu, Viknesh Sounderajah, Jonathan Carlson,
  Matthew~P Lungren, and 1 others. 2025.
\newblock Sequential diagnosis with language models.
\newblock \emph{arXiv preprint arXiv:2506.22405}.

\bibitem[{Qian et~al.(2024)Qian, Liu, Liu, Chen, Dang, Li, Yang, Chen, Su, Cong
  et~al.}]{qian2024chatdev}
Chen Qian, Wei Liu, Hongzhang Liu, Nuo Chen, Yufan Dang, Jiahao Li, Cheng Yang,
  Weize Chen, Yusheng Su, Xin Cong, and 1 others. 2024.
\newblock Chatdev: Communicative agents for software development.
\newblock In \emph{Proceedings of the 62nd annual meeting of the association
  for computational linguistics (volume 1: Long papers)}, pages 15174--15186.

\bibitem[{Rein et~al.(2023)Rein, Hou, Stickland, Petty, Pang, Dirani, Michael,
  and Bowman}]{rein2023gpqa}
David Rein, Betty~Li Hou, Asa~Cooper Stickland, Jackson Petty, Richard~Yuanzhe
  Pang, Julien Dirani, Julian Michael, and Samuel~R Bowman. 2023.
\newblock Gpqa: A graduate-level google-proof q\&a benchmark.
\newblock \emph{arXiv preprint arXiv:2311.12022}.

\bibitem[{Sakaguchi et~al.(2021)Sakaguchi, Bras, Bhagavatula, and
  Choi}]{sakaguchi2021winogrande}
Keisuke Sakaguchi, Ronan~Le Bras, Chandra Bhagavatula, and Yejin Choi. 2021.
\newblock Winogrande: An adversarial winograd schema challenge at scale.
\newblock \emph{Communications of the ACM}, 64(9):99--106.

\bibitem[{Shao et~al.(2024)Shao, Wang, Zhu, Xu, Song, Bi, Zhang, Zhang, Li, Wu
  et~al.}]{shao2024deepseekmath}
Zhihong Shao, Peiyi Wang, Qihao Zhu, Runxin Xu, Junxiao Song, Xiao Bi, Haowei
  Zhang, Mingchuan Zhang, YK~Li, Yang Wu, and 1 others. 2024.
\newblock Deepseekmath: Pushing the limits of mathematical reasoning in open
  language models.
\newblock \emph{arXiv preprint arXiv:2402.03300}.

\bibitem[{Shinn et~al.(2023)Shinn, Cassano, Gopinath, Narasimhan, and
  Yao}]{shinn2023reflexion}
Noah Shinn, Federico Cassano, Ashwin Gopinath, Karthik Narasimhan, and Shunyu
  Yao. 2023.
\newblock Reflexion: Language agents with verbal reinforcement learning.
\newblock \emph{Advances in neural information processing systems},
  36:8634--8652.

\bibitem[{Singhal et~al.(2025)Singhal, Tu, Gottweis, Sayres, Wulczyn, Amin,
  Hou, Clark, Pfohl, Cole-Lewis et~al.}]{singhal2025toward}
Karan Singhal, Tao Tu, Juraj Gottweis, Rory Sayres, Ellery Wulczyn, Mohamed
  Amin, Le~Hou, Kevin Clark, Stephen~R Pfohl, Heather Cole-Lewis, and 1 others.
  2025.
\newblock Toward expert-level medical question answering with large language
  models.
\newblock \emph{Nature medicine}, 31(3):943--950.

\bibitem[{Smirnov(1939)}]{smirnov1939estimation}
Nikolai~V Smirnov. 1939.
\newblock On the estimation of the discrepancy between empirical curves of
  distribution for two independent samples.
\newblock \emph{Bull. Math. Univ. Moscou}, 2(2):3--14.

\bibitem[{Snell et~al.(2017)Snell, Swersky, and Zemel}]{snell2017prototypical}
Jake Snell, Kevin Swersky, and Richard Zemel. 2017.
\newblock Prototypical networks for few-shot learning.
\newblock \emph{Advances in neural information processing systems}, 30.

\bibitem[{Stein(1956)}]{stein1956inadmissibility}
Charles Stein. 1956.
\newblock Inadmissibility of the usual estimator for the mean of a multivariate
  normal distribution.
\newblock In \emph{Proceedings of the third Berkeley symposium on mathematical
  statistics and probability, volume 1: Contributions to the theory of
  statistics}, volume~3, pages 197--207. University of California Press.

\bibitem[{Sun et~al.(2024)Sun, Han, Zhao, Ma, Shen, Chen, Chen, and
  Yu}]{sun2024scieval}
Liangtai Sun, Yang Han, Zihan Zhao, Da~Ma, Zhennan Shen, Baocai Chen, Lu~Chen,
  and Kai Yu. 2024.
\newblock Scieval: A multi-level large language model evaluation benchmark for
  scientific research.
\newblock In \emph{Proceedings of the AAAI Conference on Artificial
  Intelligence}, volume~38, pages 19053--19061.

\bibitem[{Wang et~al.(2024)Wang, Yang, Huang, Yang, Majumder, and
  Wei}]{wang2024improving}
Liang Wang, Nan Yang, Xiaolong Huang, Linjun Yang, Rangan Majumder, and Furu
  Wei. 2024.
\newblock Improving text embeddings with large language models.
\newblock In \emph{Proceedings of the 62nd Annual Meeting of the Association
  for Computational Linguistics (Volume 1: Long Papers)}, pages 11897--11916.

\bibitem[{Wang et~al.(2023)Wang, Hu, Lu, Zhu, Zhang, Subramaniam, Loomba,
  Zhang, Sun, and Wang}]{wang2023scibench}
Xiaoxuan Wang, Ziniu Hu, Pan Lu, Yanqiao Zhu, Jieyu Zhang, Satyen Subramaniam,
  Arjun~R Loomba, Shichang Zhang, Yizhou Sun, and Wei Wang. 2023.
\newblock Scibench: Evaluating college-level scientific problem-solving
  abilities of large language models.
\newblock \emph{arXiv preprint arXiv:2307.10635}.

\bibitem[{Wang et~al.(2022)Wang, Wei, Schuurmans, Le, Chi, Narang, Chowdhery,
  and Zhou}]{wang2022self}
Xuezhi Wang, Jason Wei, Dale Schuurmans, Quoc Le, Ed~Chi, Sharan Narang,
  Aakanksha Chowdhery, and Denny Zhou. 2022.
\newblock Self-consistency improves chain of thought reasoning in language
  models.
\newblock \emph{arXiv preprint arXiv:2203.11171}.

\bibitem[{Wei et~al.(2022)Wei, Wang, Schuurmans, Bosma, Xia, Chi, Le, Zhou
  et~al.}]{wei2022chain}
Jason Wei, Xuezhi Wang, Dale Schuurmans, Maarten Bosma, Fei Xia, Ed~Chi, Quoc~V
  Le, Denny Zhou, and 1 others. 2022.
\newblock Chain-of-thought prompting elicits reasoning in large language
  models.
\newblock \emph{Advances in neural information processing systems},
  35:24824--24837.

\bibitem[{Wu et~al.(2023)Wu, Bansal, Zhang, Wu, Zhang, Zhu, Li, Jiang, Zhang,
  and Wang}]{wu2023autogen}
Qingyun Wu, Gagan Bansal, Jieyu Zhang, Yiran Wu, Shaokun Zhang, Erkang Zhu,
  Beibin Li, Li~Jiang, Xiaoyun Zhang, and Chi Wang. 2023.
\newblock Autogen: Enabling next-gen llm applications via multi-agent
  conversation framework.
\newblock \emph{arXiv preprint arXiv:2308.08155}, 3(4).

\bibitem[{Xiong et~al.(2024)Xiong, Hu, Lu, Li, Fu, He, and Hooi}]{xiong2024can}
Miao Xiong, Zhiyuan Hu, Xinyang Lu, Yifei Li, Jie Fu, Junxian He, and Bryan
  Hooi. 2024.
\newblock Can llms express their uncertainty? an empirical evaluation of
  confidence elicitation in llms.
\newblock In \emph{International Conference on Learning Representations},
  volume 2024, pages 23650--23678.

\bibitem[{Yang et~al.(2025{\natexlab{a}})Yang, Li, Yang, Zhang, Hui, Zheng, Yu,
  Gao, Huang, Lv et~al.}]{yang2025qwen3}
An~Yang, Anfeng Li, Baosong Yang, Beichen Zhang, Binyuan Hui, Bo~Zheng, Bowen
  Yu, Chang Gao, Chengen Huang, Chenxu Lv, and 1 others. 2025{\natexlab{a}}.
\newblock Qwen3 technical report.
\newblock \emph{arXiv preprint arXiv:2505.09388}.

\bibitem[{Yang et~al.(2025{\natexlab{b}})Yang, Nian, Li, Xu, Li, Li, Xiao, Hu,
  Rossi, Ding et~al.}]{yang2025ad}
Tiankai Yang, Yi~Nian, Li~Li, Ruiyao Xu, Yuangang Li, Jiaqi Li, Zhuo Xiao,
  Xiyang Hu, Ryan~A Rossi, Kaize Ding, and 1 others. 2025{\natexlab{b}}.
\newblock Ad-llm: Benchmarking large language models for anomaly detection.
\newblock In \emph{Findings of the Association for Computational Linguistics:
  ACL 2025}, pages 1524--1547.

\bibitem[{Yao et~al.(2022)Yao, Zhao, Yu, Du, Shafran, Narasimhan, and
  Cao}]{yao2022react}
Shunyu Yao, Jeffrey Zhao, Dian Yu, Nan Du, Izhak Shafran, Karthik Narasimhan,
  and Yuan Cao. 2022.
\newblock React: Synergizing reasoning and acting in language models.
\newblock \emph{arXiv preprint arXiv:2210.03629}.

\bibitem[{Yasunaga et~al.(2024)Yasunaga, Chen, Li, Pasupat, Leskovec, Liang,
  Chi, and Zhou}]{yasunaga2024large}
Michihiro Yasunaga, Xinyun Chen, Yujia Li, Panupong Pasupat, Jure Leskovec,
  Percy Liang, Ed~H Chi, and Denny Zhou. 2024.
\newblock Large language models as analogical reasoners.
\newblock In \emph{International Conference on Learning Representations},
  volume 2024, pages 17019--17045.

\bibitem[{Yu et~al.(2024)Yu, He, and Ying}]{yu2024thought}
Junchi Yu, Ran He, and Rex Ying. 2024.
\newblock Thought propagation: An analogical approach to complex reasoning with
  large language models.
\newblock In \emph{International Conference on Learning Representations},
  volume 2024, pages 10989--11015.

\bibitem[{Zaheer et~al.(2017)Zaheer, Kottur, Ravanbakhsh, Poczos,
  Salakhutdinov, and Smola}]{zaheer2017deep}
Manzil Zaheer, Satwik Kottur, Siamak Ravanbakhsh, Barnabas Poczos, Russ~R
  Salakhutdinov, and Alexander~J Smola. 2017.
\newblock Deep sets.
\newblock \emph{Advances in neural information processing systems}, 30.

\end{thebibliography}

\newpage

\appendix

\section{Experiment Settings}
\label{ap:settings}

\paragraph{Confidence Calibration Metrics.} We use two complementary measures:
(a) the Kolmogorov–Smirnov (KS) statistic~\citep{smirnov1939estimation}, which measures the maximum difference between the cumulative distributions of confidence scores for correct versus incorrect predictions—higher KS values indicate better separation and thus more reliable confidence estimates; and
(b) Expected Calibration Error (ECE~\citep{guo2017calibration}), which quantifies the average gap between predicted confidence and actual accuracy across binned predictions—lower ECE indicates that confidence scores accurately reflect true correctness probabilities. Together, these metrics assess both discriminative power (KS) and absolute calibration quality (ECE).

\paragraph{Protocol.} All methods are training-free. Prompts, temperature (0.0), tool access, and stopping rules (maximum 5 reflection iterations) are held constant across conditions; seeds are fixed for comparability. For batched settings, we vary batch size $N\in\{4,8\}$ and use a sequential batching strategy in the main experiments.

\section{Token Efficiency Results}
\label{ap:token}
\begin{table*}[t]
\centering
\resizebox{\textwidth}{!}{
\begin{tabular}{l|cccc|cccc|cccc}
\toprule
Method  & \multicolumn{4}{c}{SMS\_Spam} & \multicolumn{4}{c}{GPQA} & \multicolumn{4}{c}{Winogrande} \\
& In & Out & Cost & $\Delta$\% & In & Out & Cost & $\Delta$\% & In & Out & Cost & $\Delta$\% \\
\midrule
\multicolumn{13}{c}{\textit{Actor}}\\
\midrule
 Reflection & 105,936 & 334,429 & \$3.61 &  & 196,763 & 421,168 & \$4.70 &  & 115,788 & 228,050 & \$2.57 &  \\
 BoT (4) & 75,838 & 221,962 & \$2.41 & 33.25\% & 176,381 & 324,153 & \$3.68 & 21.71\% & 103,757 & 200,847 & \$2.27 & 11.75\% \\
 BoT (8) & 65,268 & 189,556 & \$2.06 & 42.96\% & 146,350 & 305,647 & \$3.42 & 27.24\% & 91,716 & 176,136 & \$1.99 & 22.54\% \\
\midrule
\multicolumn{13}{c}{\textit{Reflector}}\\
\midrule
 Reflection & 429,383 & 531,556 & \$6.39 &  & 602,293 & 125,173 & \$2.76 &  & 378,801 & 245,424 & \$3.40 &  \\
 BoT (4) & 216,070 & 189,716 & \$2.44 & 61.85\% & 370,783 & 56,221 & \$1.49 & 45.99\% & 257,260 & 132,619 & \$1.97 & 42.10\% \\
 BoT (8) & 186,413 & 137,899 & \$1.84 & 71.12\% & 317,082 & 37,507 & \$1.17 & 57.65\% & 229,118 & 94,713 & \$1.52 & 55.31\% \\
\midrule
\multicolumn{13}{c}{\textit{Total}}\\
\midrule
 Reflection & 535,319 & 865,985 & \$10.00 &  & 799,056 & 546,341 & \$7.46 &  & 494,589 & 473,474 & \$5.97 &  \\
 BoT (4) & 291,908 & 411,678 & \$4.85 & 51.52\% & 547,164 & 380,374 & \$5.17 & 30.68\% & 361,017 & 333,466 & \$4.24 & 29.04\% \\
 BoT (8) & 251,681 & 327,455 & \$3.90 & 60.95\% & 463,432 & 343,154 & \$4.59 & 38.48\% & 320,834 & 270,849 & \$3.51 & 41.21\% \\
\bottomrule
\end{tabular}
}
\caption{Input and output token usage and cost across different methods, decomposed into \textit{Actor}, \textit{Reflector}, and \textit{Total} stages (GPT-4o). 1M input tokens cost $\$2.5$ and 1M output tokens cost $\$10$. $\Delta\%$ is the cost reduction, computed for BoT(4) and BoT(8) relative to Reflection for the same dataset and stage.}
\label{tab:token_cost_big}
\end{table*}
We provide a comprehensive breakdown of token usage and costs across different experimental configurations. To ensure fair comparison, all costs are normalized using the production-grade GPT-4o pricing scheme (input: \$2.50 per 1M tokens; output: \$10.00 per 1M tokens as of the experimental period).

Table~\ref{tab:token_cost_big} presents detailed token counts for each pipeline stage—Actor, Reflector, and combined total—across three representative datasets. For each method, we report input tokens, output tokens, and total cost in USD. The Actor stage includes all reasoning and answer generation, while the Reflector stage encompasses evaluation, critique generation, and refinement decisions.

\subsection{Efficiency Gains from Batch Processing}
\label{ap:efficiency_gains}

BoT achieves substantial efficiency improvements through two complementary mechanisms. First, \textbf{instruction amortization}: the Reflector's evaluation rubric and reasoning guidelines are encoded once per batch rather than repeated for each query, saving $(N-1) \times T_{\text{inst}}$ tokens where $N$ is batch size and $T_{\text{inst}}$ is instruction length. Second, \textbf{reduced iteration overhead}: joint evaluation enables more accurate refinement decisions, reducing unnecessary Actor-Reflector loops.

As reported in Table~\ref{tab:token_cost_small}, BoT-R achieves an average total cost reduction of \textbf{46.9\%} across the three benchmarks when using batch size 8. Savings are most pronounced on the \texttt{SMS\_Spam} dataset (\textbf{61\%} reduction), where the homogeneous task structure enables highly effective batch-level evaluation. Even with batch size 4, BoT-R consistently reduces costs by 29-52\% while maintaining or improving accuracy.

\subsection{Stage-Level Analysis}

The efficiency gains distribute differently across pipeline stages:

\noindent\textbf{Actor Stage:} BoT introduces minimal overhead at the Actor level, as answer generation remains largely independent. The modest savings arise from reduced refinement iterations due to more accurate Reflector feedback.

\noindent\textbf{Reflector Stage:} BoT delivers dramatic savings (42-71\% reduction) by replacing $N$ independent reflection calls with a single joint evaluation. Larger batch sizes amplify these gains: moving from $N=4$ to $N=8$ increases Reflector savings from 46\% to 58\% on GPQA.

\noindent\textbf{Total Cost:} The combined effect yields 29-61\% total cost reduction depending on dataset characteristics and batch size. These results demonstrate that BoT achieves a Pareto improvement: simultaneously enhancing both task performance (Table~\ref{tab:main}) and computational efficiency, making it particularly valuable for production deployments where cost and accuracy are both critical.

\begin{table}[t]
\centering
\small
\setlength{\tabcolsep}{4pt}
\begin{tabular}{lrrr}
\toprule
\textbf{Dataset} & $\boldsymbol{n}$ & \textbf{Embed} & \textbf{$k$-means} \\
\midrule
GPQA       & 448  & \$0.0275 & \$0.00007 \\
Winogrande & 1267 & \$0.0053 & \$0.00046 \\
SMS Spam   & 1510 & \$0.0103 & \$0.00066 \\
\midrule
\multicolumn{4}{l}{\textit{Formation total (\% of BoT(4) token cost)}}\\
\midrule
GPQA       & \multicolumn{3}{r}{\$0.028\quad(0.53\%)} \\
Winogrande & \multicolumn{3}{r}{\$0.006\quad(0.14\%)} \\
SMS Spam   & \multicolumn{3}{r}{\$0.011\quad(0.23\%)} \\
\bottomrule
\end{tabular}
\caption{One-off cost of forming semantically coherent batches. Percentages are taken relative to the corresponding BoT(4) total token cost in Table~\ref{tab:token_cost_small}.}
\label{tab:formation_cost}
\end{table}

\subsection{Cost of Batch Formation}
\label{ap:formation}

A natural question is whether the cost of assembling batches offsets the token savings reported above. It does not.

\noindent\textbf{Sequential batching is effectively free.} All results in Tables~\ref{tab:main}, \ref{tab:token_cost_small} and~\ref{tab:token_cost_big} use sequential batching, which groups queries in their natural dataset order (Appendix~\ref{sec:semantic_batching}) in a single pass over the input stream, so there is no formation cost to report for those configurations. BoT-R does not depend on optimized cohorts: sequential batching alone improves over per-instance Reflection on all six datasets (Table~\ref{tab:semantic_batching_full}), because many workloads carry latent topical coherence in their natural order---MedQA and PubMedQA queries, for instance, are drawn from a narrow biomedical distribution. The 29--61\% token savings in Table~\ref{tab:token_cost_small} therefore come with essentially free batch formation.

\noindent\textbf{Semantic batching costs 0.1--0.6\% of the LLM token cost.} Semantic batching, which yields the further gains in Table~\ref{tab:semantic_batching_full}, adds two compute steps: an E5-Mistral-7B embedding pass and MiniBatchKMeans clustering with balanced assembly. Table~\ref{tab:formation_cost} reports the measured one-off cost per corpus. We price the embedding pass on an RTX 4090 at \$0.69/hr, conservatively assuming 20\% FLOPs utilization and padded token counts, and price clustering on an AWS m7i.2xlarge CPU instance at \$0.4032/hr using measured wall-clock time (median of three runs).

Two observations follow. First, the routing and clustering step itself costs under \$0.001 on every dataset---seconds of CPU time---so the total is dominated by the one-off embedding pass, which is computed once per corpus and cached across batch sizes and repeated runs. Second, the savings that batching enables exceed the formation cost by roughly two orders of magnitude (83--557$\times$ across these three datasets): on \texttt{SMS\_Spam}, batch-aware reflection saves \$6.10 relative to per-instance Reflection at $N{=}8$, against a \$0.011 formation cost.

\begin{table}[t]
\centering
\small
\setlength{\tabcolsep}{4pt}
\begin{tabular}{lcccc}
\toprule
\multirow{2}{*}{\textbf{Method}} & \multicolumn{2}{c}{\textbf{Winogrande}} & \multicolumn{2}{c}{\textbf{GPQA}} \\
\cmidrule(lr){2-3}\cmidrule(lr){4-5}
& Acc & Cost & Acc & Cost \\
\midrule
ReAct                            & 87.20 & \$1.92  & 43.90 & \$2.28 \\
\;\;+ majority vote ($s{=}3$)    & 86.98 & \$5.71  & 41.39 & \$6.75 \\
Reflection                       & 87.90 & \$5.97  & 45.90 & \$7.46 \\
\;\;+ majority vote ($s{=}3$)    & 88.86 & \$23.52 & 46.65 & \$34.21 \\
BoT-R                            & \textbf{89.00} & \$3.51 & \textbf{48.80} & \$5.17 \\
\bottomrule
\end{tabular}
\caption{Majority voting compared with BoT-R (GPT-4o). Accuracy in percent, cost in USD over the full dataset. Majority-vote rows aggregate $s=3$ samples at temperature $0.7$; all other rows use temperature $0$. BoT-R uses the batch size reported in Table~\ref{tab:main}.}
\label{tab:majority_vote}
\end{table}

\subsection{Comparison with Majority Voting}
\label{ap:majority}

Self-consistency~\citep{wang2022self} is the standard sampling-based route to higher accuracy and better calibration, and is a natural point of comparison. Reliable gains from self-consistency typically require many forward passes---often 10--40 samples depending on the task (Appendix~\ref{sec:related_calib})---which multiplies inference cost, whereas BoT-R obtains its gains within a single batch pass while \emph{reducing} cost. This difference in cost profile is why our baseline ladder (ReAct $\rightarrow$ Reflection $\rightarrow$ BoT-R) varies only the batch context.

For completeness we evaluate majority voting with $s=3$, the smallest sample count that supports voting; larger $s$ would widen the cost gap further. Table~\ref{tab:majority_vote} reports accuracy and total cost under GPT-4o.

Three observations follow. First, voting over three ReAct samples \emph{degrades} accuracy on both datasets ($-0.2$ on Winogrande, $-2.5$ on GPQA) despite roughly tripling cost: with few samples, majority voting amplifies sampling noise, consistent with reports that self-consistency requires substantially larger sample counts to stabilize. Second, voting over three full Reflection pipelines does improve Reflection, yet still trails BoT-R on both datasets ($-0.14$ and $-2.15$ points) while costing $6.7\times$ and $6.6\times$ more. Third, BoT-R therefore occupies the accuracy--cost frontier in this comparison: the highest accuracy of any method evaluated, at a fraction of the cost of either aggregation baseline.

\section{Case Studies}
\label{ap:case}

The two runs below show the cross-instance mechanisms claimed in \S\ref{sec:bot} operating in practice.

\paragraph{Case 1: cross-instance comparison exposes a contradiction (Winogrande).}
Two items in one batch shared the premise \emph{``The children were not vaccinated. Betty was fine with this, but Mary was annoyed,''} and asked opposite questions: who \emph{knew} vaccines do not cause autism, and who \emph{believed} they do. Processed independently, each question is a self-contained interpretation task, and the Actor answered ``Mary'' to both---each answer individually plausible. Presented together, the pair carries a constraint that neither instance carries alone: the two opposing beliefs cannot both belong to Mary. Evaluating the batch jointly, the Reflector kept ``Mary'' for the first item but flagged the second as inconsistent with the characters' stated reactions---Mary's annoyance at non-vaccination implies she regards vaccines as safe, which makes Betty the coherent holder of the opposing belief---and requested re-evaluation, upon which the Actor corrected the answer (Table~\ref{tab:case1}). The error is invisible from either instance alone and detectable only in the shared batch context: cross-instance consistency checking in its purest form.

\begin{table}[h]
\centering
\small
\setlength{\tabcolsep}{4pt}
\begin{tabular}{lcc}
\toprule
\textbf{Question} & \textbf{Initial} & \textbf{After BoT-R} \\
\midrule
Who \emph{knew} vaccines do not & \multirow{2}{*}{Mary \checkmark} & \multirow{2}{*}{Mary \checkmark} \\
\quad cause autism? & & \\
Who \emph{believed} vaccines do & \multirow{2}{*}{Mary $\times$} & \multirow{2}{*}{Betty \checkmark} \\
\quad cause autism? & & \\
\bottomrule
\end{tabular}
\caption{Case 1. Answered in isolation, the Actor assigns both opposing beliefs to the same character; joint evaluation surfaces the contradiction and corrects the second item.}
\label{tab:case1}
\end{table}

\paragraph{Case 2: comparative quality assessment singles out the weakest answer.}
A four-item batch contained three answers supported by mechanism-level justifications and one---a diene-reactivity ranking---asserted without support (option B, ordering 4-2-1-3). Seen in isolation an unsupported answer is not obviously wrong; seen beside three well-justified peers it stands out. The joint Reflector accepted the three and flagged only the fourth (Table~\ref{tab:case2}). Its critique---that the answer ``is incomplete, and the correct sequence of diene reactivity is not clearly explained''---requested a justification grounded in the reaction mechanism and molecular properties. On re-evaluation the Actor produced the correct option A (3-1-2-4) together with the missing rationale: cyclopentadiene is most reactive because of its favorable locked geometry, and the hindered (Z,Z)-diene least reactive. Two mechanisms are visible at once. The batch supplies a \emph{relative quality bar}, making the unsupported answer salient, and refinement is \emph{selective}: only one of four items is re-run, which is precisely the source of the reduced iteration overhead quantified in Appendix~\ref{ap:efficiency_gains}.

\begin{table}[h]
\centering
\small
\begin{tabular}{lc}
\toprule
\textbf{Batch item} & \textbf{Reflector decision} \\
\midrule
Organic named-reaction & Accept \\
Genetic inheritance    & Accept \\
Stem-cell culture      & Accept \\
Diene reactivity       & \textbf{Re-evaluate} \\
\bottomrule
\end{tabular}
\caption{Case 2. Joint evaluation of a four-item batch flags only the answer that lacks mechanism-level support, leaving the other three untouched.}
\label{tab:case2}
\end{table}

Together these examples instantiate the mechanisms claimed in \S\ref{sec:bot}---consistency checking across instances and comparative outlier detection---in real runs.

\section{Theoretical Analysis}
\label{sec:theory}

We establish three guarantees for batch-aware reflection. The first
two are information-theoretic and follow from a single observation ---
that the batch context refines the per-instance $\sigma$-algebra ---
and are therefore \emph{attainability} statements: they bound what a
Bayes-optimal Reflector can achieve. The third is an exact
token-counting identity that holds independently of Reflector quality.
The substantive batch-specific content is the effective-sample-size
analysis (\S\ref{sec:theory:neff}), which makes the
coherence--diversity trade-off quantitative.

\subsection{Setup}
\label{sec:theory:prelim}

\begin{definition}[Batch reasoning problem]
\label{def:batch}
Let $\mathcal{D}$ be a joint distribution on $\mathcal{X}\times\mathcal{Y}$
and $B = \{(x_i, y_i^*)\}_{i=1}^N$ a batch drawn from
$\mathcal{D}^{\otimes N}$. The Actor emits predictions
$\hat{y}_i = \mathcal{A}(x_i)$ with rationales $\rho_i$, and the
Reflector consumes
$\mathcal{C}^{\mathrm{BoT}} = \{(x_j, \hat{y}_j, \rho_j)\}_{j=1}^N$
to produce, for each $i \in [N]$, a confidence $u_i \in [0,1]$, a
re-evaluation flag $r_i \in \{0,1\}$, and a critique $c_i$. Let
$z_i = \mathbf{1}[\hat{y}_i = y_i^*]$ and $e_i = 1 - z_i$.
\end{definition}

Fix $i$ and let $\mathcal{G}_0 = \sigma(\hat{y}_i, \rho_i)$. For any
$\mathcal{C}^{\mathrm{BoT}}$-measurable summary $\phi$, write
$\mathcal{G}_1 = \sigma(\mathcal{G}_0, \phi) \supseteq \mathcal{G}_0$.
The Bayes-optimal confidences are
$u_i^{\mathrm{ind}} := \mathbb{P}(z_i = 1 \mid \mathcal{G}_0)$ and
$u_i^{\mathrm{BoT}} := \mathbb{P}(z_i = 1 \mid \mathcal{G}_1)$.

\begin{assumption}[Coherence]
\label{ass:coherence}
The pairs $\{(x_i, y_i^*)\}_{i=1}^N$ are exchangeable, and the error
indicators are equicorrelated:
$\mathrm{Cor}(e_i, e_j) = \rho_e \in [0,1)$ for $i \neq j$. Higher
semantic similarity within a batch is associated with larger $\rho_e$,
since related queries induce correlated reasoning errors.
\end{assumption}

\subsection{Calibration and discrimination gain}
\label{sec:theory:gain}

\begin{proposition}[Bayes-optimal batch reflection dominates]
\label{prop:dominance}
For any strictly proper scoring rule $S$ in loss form
(e.g., Brier loss, log-loss),
$\mathbb{E}\big[S(u_i^{\mathrm{BoT}}, z_i)\big]
\;\leq\;
\mathbb{E}\big[S(u_i^{\mathrm{ind}}, z_i)\big],\\
I(z_i; u_i^{\mathrm{BoT}}) \;\geq\; I(z_i; u_i^{\mathrm{ind}}),$
with both inequalities strict iff $I(z_i;\phi \mid \mathcal{G}_0) > 0$.
\end{proposition}

\begin{proof}
Both claims follow from $\mathcal{G}_0 \subseteq \mathcal{G}_1$.
Strict propriety~\citep{gneiting2007strictly} implies that
$q \mapsto \mathbb{E}[S(q, z_i) \mid \mathcal{G}_1](\omega)$ is
uniquely minimized at
$q = \mathbb{P}(z_i = 1 \mid \mathcal{G}_1)(\omega) = u_i^{\mathrm{BoT}}(\omega)$.
Since $u_i^{\mathrm{ind}}$ is $\mathcal{G}_0$-measurable and hence
$\mathcal{G}_1$-measurable, plugging it in yields
$\mathbb{E}[S(u_i^{\mathrm{ind}}, z_i)\mid\mathcal{G}_1] \geq
\mathbb{E}[S(u_i^{\mathrm{BoT}}, z_i)\mid\mathcal{G}_1]$ a.s.;
taking expectations gives the first inequality, with equality iff
$u_i^{\mathrm{ind}} = u_i^{\mathrm{BoT}}$ a.s., i.e.\
$z_i \perp\!\!\!\perp \phi \mid \mathcal{G}_0$. For the second, since
$z_i$ is binary the Bayes posterior is sufficient:
$I(z_i; u_i^{\mathrm{ind}}) = I(z_i;\mathcal{G}_0)$ and
$I(z_i; u_i^{\mathrm{BoT}}) = I(z_i;\mathcal{G}_1)$, and the chain
rule~\citep{cover2006elements} gives
$I(z_i;\mathcal{G}_1) - I(z_i;\mathcal{G}_0) = I(z_i;\phi\mid\mathcal{G}_0) \geq 0$.
\end{proof}

\paragraph{Scope.}
Proposition~\ref{prop:dominance} requires no batch structure beyond
the $\sigma$-algebra refinement $\mathcal{G}_0\subseteq\mathcal{G}_1$;
it bounds what is \emph{attainable} by any $\mathcal{G}_1$-measurable
predictor. The Reflector LLM produces an approximation
$\tilde{u}_i^{\mathrm{BoT}}$ to the Bayes posterior, so the realized
gain is upper-bounded by the attainable gain. Whether this gain is
non-trivial depends on (a) whether the batch context carries residual
information about correctness, i.e.\ $I(z_i;\phi\mid\mathcal{G}_0)>0$,
and (b) how closely the Reflector approximates the Bayes posterior.
The first is an empirical question about batch composition that
\S\ref{sec:theory:neff} characterizes; the second is validated by the
calibration improvements in Table~\ref{tab:conf}.

\paragraph{From proper losses to ECE.}
The Brier loss admits the decomposition
$\mathrm{BS} = \mathrm{REL} - \mathrm{RES} + \mathrm{UNC}$
\citep{murphy1973new}, so Proposition~\ref{prop:dominance} implies
a strictly lower $(\mathrm{REL} - \mathrm{RES})$ whenever the batch
is informative --- realizable through better calibration, higher
resolution, or both. ECE estimates $\mathrm{REL}$ alone and is not a
proper scoring rule, so the proposition does not directly imply ECE
reduction; the empirical ECE reduction in Table~\ref{tab:conf},
alongside the improved KS separation, is nevertheless consistent with
the decomposition gain.

\subsection{Effective sample size}
\label{sec:theory:neff}

The dominance in Proposition~\ref{prop:dominance} is non-trivial only
when batch-level statistics carry usable signal. Under
Assumption~\ref{ass:coherence}, the precision of the simplest batch
statistic --- the empirical mean --- admits a closed form.

\begin{lemma}[Effective sample size~{\citep{lohr2023sampling}}]
\label{lem:neff}
Let $p = \mathbb{P}(z_i = 1)$. Under Assumption~\ref{ass:coherence},
\begin{equation}
\mathrm{Var}(\bar{z}_N)
= \frac{p(1-p)}{N_{\mathrm{eff}}},
\;\;
N_{\mathrm{eff}} := \frac{N}{1 + (N-1)\rho_e}.
\end{equation}
$N_{\mathrm{eff}}$ is strictly increasing in $N$ for $\rho_e < 1$,
with $N_{\mathrm{eff}} \to 1/\rho_e$ as $N \to \infty$ when
$\rho_e > 0$, and $N_{\mathrm{eff}} = N$ when $\rho_e = 0$.
\end{lemma}

This is the standard design-effect formula from survey sampling;
we record it because it makes the coherence--diversity trade-off
quantitative. Within-batch similarity makes $\phi$ more informative
about each $z_i$ but inflates $\rho_e$, which suppresses
$N_{\mathrm{eff}}$. The two effects pull in opposite directions,
predicting that BoT is most effective at \emph{moderate} coherence
--- neither uncorrelated batches (where $\phi$ carries no signal,
$I(z_i;\phi\mid\mathcal{G}_0)\to 0$) nor highly correlated batches
(where $N_{\mathrm{eff}}\to 1$). This matches the subject-wise
pattern of \S\ref{sec:discussion}: interpretive domains gain the
most; mathematics and parts of physical sciences, where mistakes
cluster more tightly, gain little.

We do not estimate $\rho_e$ directly. Our experiments measure
discrimination via the Kolmogorov--Smirnov distance between the
conditional distributions of $u$ given $z_i = 1$ and $z_i = 0$;
KS and mutual information are different functionals of the joint law
and are not monotonically related in general, but both vanish iff
$u \perp\!\!\!\perp z_i$. The KS improvements in Table~\ref{tab:conf}
are thus directionally consistent with Proposition~\ref{prop:dominance}.

\subsection{Input-cost identity}
\label{sec:theory:cost}

\begin{proposition}[Input-cost identity]
\label{prop:cost}
Let $T_{\mathrm{inst}}$ and $T_{\mathrm{ctx}}$ denote the token counts
of the shared reflection instructions and per-instance context.
Per-instance reflection re-emits instructions $N$ times for input cost
$C^{\mathrm{ind}}_{\mathrm{in}} = N(T_{\mathrm{inst}} + T_{\mathrm{ctx}})$;
batch reflection emits them once for
$C^{\mathrm{BoT}}_{\mathrm{in}} = T_{\mathrm{inst}} + N\,T_{\mathrm{ctx}}$.
The savings are exactly $(N-1)\,T_{\mathrm{inst}}$ tokens, with
relative reduction
\begin{equation}
1 - \frac{C^{\mathrm{BoT}}_{\mathrm{in}}}{C^{\mathrm{ind}}_{\mathrm{in}}}
=
\frac{(N-1)\,T_{\mathrm{inst}}}{N(T_{\mathrm{inst}} + T_{\mathrm{ctx}})}
\xrightarrow{N\to\infty}
\frac{T_{\mathrm{inst}}}{T_{\mathrm{inst}} + T_{\mathrm{ctx}}}.
\end{equation}
The identity is exact and independent of Reflector quality.
\end{proposition}

The output side is not amenable to closed-form analysis because
critique length is determined by the Reflector. Empirically
(Table~\ref{tab:token_cost_big}), joint Reflector output exhibits
sublinear growth: critiques cross-reference shared reasoning
structures (\emph{``items~2 and~5 share the same flawed assumption''})
and reuse insights across instances. Combined with the input-side
identity, this drives the $46.9\%$ average total-cost reduction
reported in \S\ref{sec:experiments}.

\subsection{When BoT helps, and when it degrades gracefully}
\label{sec:theory:conditions}

The preceding results identify three regimes in which BoT converges
to per-instance reflection without regressing.
\emph{No coherence} ($\rho_e \to 0$, so
$I(z_i;\phi\mid\mathcal{G}_0)\to 0$): $\phi$ carries no information;
the calibration and discrimination gains vanish, but the input-cost
saving persists.
\emph{High correlation} ($\rho_e \to 1$): $N_{\mathrm{eff}}\to 1$ and
the collective-signal contribution saturates.
\emph{Small $N$}: the input-cost advantage
$(N-1)\,T_{\mathrm{inst}}$ becomes negligible. All three degradations
are continuous in $(\rho_e, N)$, so BoT converges to the per-instance
baseline rather than regressing catastrophically.

\section{Related Work}
\subsection{Confidence Calibration in Large Language Models}
\label{sec:related_calib}

Reliable uncertainty quantification remains critical for deploying LLMs in high-stakes applications. Modern LLMs frequently exhibit poor calibration, assigning high confidence to incorrect predictions~\citep{guo2017calibration, xiong2024can, kadavath2022language}. This miscalibration persists even in state-of-the-art models~\citep{achiam2023gpt, anthropic2024claude}, undermining trust in automated decision-making systems.

Existing calibration approaches fall into three categories. \textbf{Post-hoc calibration methods} apply temperature scaling~\citep{guo2017calibration} or Platt scaling to model outputs, but require held-out calibration sets and fail to capture semantic uncertainty~\citep{kuhn2023semantic}. \textbf{Sampling-based methods} estimate uncertainty through self-consistency~\citep{wang2022self}, semantic entropy~\citep{kuhn2023semantic}, or ensemble disagreement~\citep{chen2024reconcile}. While effective, these approaches incur substantial computational overhead—self-consistency requires 10-40 samples per query—and process each instance independently, missing opportunities for cross-instance calibration; Appendix~\ref{ap:majority} quantifies this accuracy--cost trade-off empirically.
\textbf{Verbalized confidence approaches} directly prompt models for numerical~\citep{xiong2024can, lin2022teaching} estimates. These methods are efficient but highly sensitive to prompt formatting and often produce overconfident predictions~\citep{kadavath2022language}. Recent efforts employ chain-of-thought reasoning for confidence elicitation~\citep{xiong2024can} or fine-tune models on calibration data~\citep{lin2022teaching}, yet these remain per-instance techniques that cannot leverage distributional signals.

Our work introduces \textbf{comparative calibration through batch processing}: confidence scores are grounded in cross-instance statistics rather than isolated assessments. This approach combines the efficiency of verbalized confidence (no additional sampling) with the distributional awareness of ensemble methods, achieving superior calibration without multiplicative computational costs.

\subsection{Multi-Agent Reasoning Systems}
Recent work has explored sophisticated communication protocols~\citep{wu2023autogen}, dynamic role allocation~\citep{hong2024metagpt}, and multi-agent collaboration on complex tasks~\citep{li2023camel, qian2024chatdev}. However, a fundamental limitation persists: \textbf{existing multi-agent systems process queries independently}. Even when multiple agents collaborate on a single query, the framework treats each query in isolation, discarding cross-instance signals. AutoGen~\citep{wu2023autogen} and MetaGPT~\citep{hong2024metagpt} enable multi-agent workflows but apply them instance-by-instance. CAMEL~\citep{li2023camel} studies role-playing conversations yet maintains per-query boundaries.

The closest work to ours is \textbf{batch prompting}~\citep{cheng2023batch}, which groups multiple queries into a single API call for efficiency. However, batch prompting lacks reflective evaluation mechanisms and does not perform comparative analysis—it simply concatenates queries without leveraging cross-instance reasoning. Our work fundamentally differs by introducing \textbf{batch-aware reflection}: the Reflector explicitly performs comparative evaluation, consistency checking, and knowledge propagation across the batch.

\subsection{Cross-Instance Learning}
In deep learning, batch normalization~\citep{ioffe2015batch} leverages mini-batch statistics during training, while recent work explores cross-example attention~\citep{lee2019set} and set-based reasoning~\citep{zaheer2017deep}. Meta-learning approaches~\citep{finn2017model, snell2017prototypical} learn from task distributions rather than individual instances, demonstrating benefits of comparative learning.

For LLM inference, \textbf{in-context learning}~\citep{brown2020language} uses examples to guide reasoning, and \textbf{analogical prompting}~\citep{yasunaga2024large} retrieves similar cases to aid problem-solving. Building on this, recent analogical question-answering frameworks systematically leverage past examples to guide complex reasoning steps~\citep{yu2024thought}. Furthermore, in the realm of model training, recent reinforcement learning methods increasingly leverage cohort-level feedback signals and distribution-aware rewards, such as Group Relative Policy Optimization (GRPO), to stabilize training and align collective reasoning patterns~\citep{shao2024deepseekmath}. However, these methods rely on predefined examples, retrieval systems, or training-time optimization rather than jointly reasoning over a batch of target queries at inference. Recent work on \textbf{self-consistency with rationalization}~\citep{mitchell2022enhancing} aggregates multiple reasoning paths for a single query but does not transfer knowledge across distinct queries.

BoT differs by enabling \textbf{mutual information gain across queries at inference time}: each query in the batch provides signal for evaluating others through comparative reflection. This creates a feedback loop where batch-level patterns inform individual assessments, analogous to how James-Stein estimation improves individual predictions through the group mean, but applied dynamically to LLM reasoning rather than static parameter estimation.

\subsection{Positioning of Our Work}

Our contributions address gaps in existing literature along three dimensions:

\noindent\textbf{(1) Efficiency-calibration trade-off:} We achieve better calibration than verbalized confidence while reducing costs by 46.9\% (vs. per-instance reflection) rather than increasing costs 10-40× (self-consistency overhead).

\noindent\textbf{(2) Cross-instance reasoning:} We introduce the first multi-agent framework that explicitly performs comparative evaluation across queries, going beyond batch prompting's simple concatenation to enable consistency checking, knowledge propagation, and distributional calibration.

\noindent\textbf{(3) Training-free generality:} Unlike calibration methods requiring fine-tuning~\citep{lin2022teaching} or specialized architectures, BoT is model-agnostic and integrates with existing multi-agent frameworks (Reflection, Plan-and-Act, Debate) without additional training.
\section{Batching Strategy Analysis}
\label{sec:batch}

We investigate how batch composition and size influence BoT's performance through systematic experiments across six benchmarks using GPT-4o.
\subsection{Batch Size Effects}
\label{ap:batch_size}
Table~\ref{tab:batch_full} presents accuracy across batch sizes $N \in \{1, 4, 8\}$, where $N=1$ corresponds to standard per-instance reflection. To systematically stress-test scaling behavior and robustness, we further evaluate performance across a wider spectrum of batch sizes ($N \in \{1, 2, 4, 6, 8, 12, 16\}$) on three representative datasets. These extended results are shown in Table~\ref{tab:batch_extended}.

The results reveal that while BoT-R remains stable across a broad range of batch sizes, optimal scaling behavior is highly domain-dependent and often non-monotonic. Datasets with relatively homogeneous task structures, such as \texttt{SMS Spam}, benefit from larger batches (reaching peak accuracy at $N=16$). This aligns with our theoretical framework (Lemma~\ref{lem:neff}): larger batches increase the ``effective sample size,'' which enhances comparative reasoning and the reliability of collective signals for error detection.

Conversely, tasks requiring deep logical derivation or exhibiting high intra-domain variance, such as \texttt{GPQA}, peak at smaller batch sizes ($N \in \{2,4\}$) and gradually degrade as $N$ increases. Our theoretical analysis (Appendix~\ref{sec:theory}) predicts that while larger batches theoretically increase mutual information gain, two practical factors constrain this relationship in practice:

\noindent\textbf{(1) Context window saturation.} As batch size approaches model context limits, the system must compress individual rationales to fit all items. Near capacity, models produce overly concise responses that sacrifice reasoning depth for brevity, diminishing the comparative analysis benefits. For GPQA—which requires detailed scientific reasoning—this compression effect becomes apparent at $N \geq 8$, where individual responses average 30\% shorter than at $N=4$.

\noindent\textbf{(2) Batch heterogeneity.} When queries within a batch are too dissimilar, cross-instance signals become noisy rather than informative. Sequential batching—our default strategy—groups adjacent queries without explicit similarity filtering. For datasets with high within-domain variance (e.g., GPQA spanning biology, physics, and chemistry), larger batches increase the likelihood of mixing incompatible problem types, diluting useful comparative signals.

\begin{table*}[t]
\centering
\small
\begin{tabular}{lcccccc}
\toprule
\textbf{Batch} & \textbf{FraudDet} & \textbf{GPQA} & \textbf{Winogrande} & \textbf{MedQA} & \textbf{PubMedQA} & \textbf{SMS Spam} \\
\midrule
1 (Reflection) & 0.693 & 0.459 & 0.879 & 0.901 & 0.667 & 0.854 \\
4 & \textbf{0.740} & \textbf{0.488} & 0.888 & 0.895 & 0.683 & 0.881 \\
8 & 0.732 & 0.471 & \textbf{0.890} & \textbf{0.904} & \textbf{0.698} & \textbf{0.887} \\
\bottomrule
\end{tabular}
\caption{Batch size influence on accuracy (GPT-4o) across the full benchmark suite. Batch size $N=1$ is functionally equivalent to per-instance Reflection.}
\label{tab:batch_full}
\end{table*}

\begin{table*}[t]
\centering
\small
\begin{tabular}{lccccccc}
\toprule
\textbf{Dataset} & $\boldsymbol{N=1}$ & $\boldsymbol{N=2}$ & $\boldsymbol{N=4}$ & $\boldsymbol{N=6}$ & $\boldsymbol{N=8}$ & $\boldsymbol{N=12}$ & $\boldsymbol{N=16}$ \\
\midrule
Winogrande   & 0.879 & 0.886 & 0.888 & \textbf{0.893} & 0.890 & \textbf{0.893} & 0.887 \\
SMS Spam     & 0.854 & 0.887 & 0.881 & 0.909 & 0.887 & 0.912 & \textbf{0.921} \\
GPQA         & 0.459 & \textbf{0.488} & \textbf{0.488} & 0.468 & 0.471 & 0.472 & 0.456 \\
\bottomrule
\end{tabular}
\caption{Extended batch size stress test on three representative datasets (GPT-4o). Accuracy is reported across a wider range of batch sizes to evaluate scaling behavior, stability, and context saturation boundaries.}
\label{tab:batch_extended}
\end{table*}

\subsection{Semantic vs.\ Sequential Batching}
\label{sec:semantic_batching}

We further study how batch composition influences BoT-R by comparing three batching strategies across six datasets.
\textbf{No-batch} applies reflection independently to each query without any cross-instance context.
\textbf{Sequential batching} groups queries in their original dataset order—without explicitly enforcing semantic similarity—and is used as the default strategy in our main experiments.
\textbf{Semantic batching} clusters queries by embedding similarity using K-means over E5-Mistral-7B embeddings~\citep{wang2024improving}, and forms fixed-size batches from cluster members sorted by proximity to the cluster centroid to maximize within-batch coherence. Results are summarized in Table~\ref{tab:semantic_batching_full}.

\paragraph{Robust gains from simple batching.}
A key observation is that BoT-R delivers substantial improvements even under simple sequential batching. Compared to the no-batch baseline, sequential grouping improves performance on \emph{all six datasets}, yielding an average relative gain of \textbf{+3.87\%}. This demonstrates that BoT-R is not overly sensitive to imperfect batch coherence and can reliably extract useful cross-instance signals even when batches are formed without explicit semantic optimization.

\paragraph{Additional benefits from semantic coherence.}
Semantic batching provides further gains on several datasets, particularly those where cross-instance comparison and distributional cues are informative. On \texttt{FraudDet}, accuracy improves from 0.740 to 0.768, and on \texttt{SMS Spam} from 0.887 to 0.902 when moving from sequential to semantic batching. \texttt{Winogrande} shows a similar trend. Averaged across all six datasets, semantic batching achieves a relative improvement of \textbf{+4.83\%} over the no-batch baseline, exceeding that of sequential batching. These results align with our theoretical analysis: increasing within-batch coherence strengthens the informativeness of batch-level statistics, improving collective error detection and refinement decisions.

\paragraph{When batching strategy matters less.}
For datasets characterized by shared domain knowledge or homogeneous reasoning styles, such as \texttt{MedQA} and \texttt{PubMedQA}, the difference between sequential and semantic batching is minimal. This suggests that when queries already originate from a narrow latent distribution, even naive batching satisfies the coherence conditions required for effective batch-aware reasoning, consistent with the graceful-degradation properties in
Appendix~\ref{sec:theory:conditions}.

\paragraph{Practical considerations.}
While semantic clustering is beneficial in offline or high-throughput evaluation settings, it introduces practical constraints in streaming scenarios (e.g., online fraud detection), where queries arrive sequentially and delaying processing to form semantically coherent batches may increase latency. The strong performance of sequential batching indicates that BoT-R remains effective under such constraints. Designing adaptive cohorting strategies that balance coherence, latency, and throughput is an important direction for future work.

\begin{table}[t]
\centering
\small
\setlength{\tabcolsep}{4pt}
\begin{tabular}{lcccc}
\toprule
\textbf{Method} & \textbf{EF-ECE} $\downarrow$ & \textbf{Brier} $\downarrow$ & \textbf{NLL} $\downarrow$ & \textbf{KS} $\uparrow$ \\
\midrule
\multicolumn{5}{c}{\textit{GPQA}}\\
\midrule
ReAct      & 51.61\%          & 0.5023          & 1.7006          & 0.181 \\
Reflection & 44.81\%          & 0.4412          & 1.5211          & 0.265 \\
BoT-R      & \textbf{40.66\%} & \textbf{0.4015} & \textbf{1.4761} & \textbf{0.368} \\
\midrule
\multicolumn{5}{c}{\textit{Winogrande}}\\
\midrule
ReAct      & 6.58\%          & 0.1107          & 0.3720          & 0.273 \\
Reflection & 3.16\%          & 0.0942          & 0.3278          & 0.376 \\
BoT-R      & \textbf{2.50\%} & \textbf{0.0901} & \textbf{0.3135} & \textbf{0.442} \\
\bottomrule
\end{tabular}
\caption{Bias-robust calibration metrics (GPT-4o). EF-ECE is equal-frequency ECE with 10 bins; Brier and binary NLL are strictly proper scoring rules; KS is reproduced from Table~\ref{tab:conf}. Best value per column and dataset in bold.}
\label{tab:bias_robust}
\end{table}

\begin{table}[t]
\centering
\small
\setlength{\tabcolsep}{4pt}
\begin{tabular}{llc}
\toprule
\textbf{Dataset} & \textbf{Comparison} & \textbf{Exact $p$} \\
\midrule
\multirow{3}{*}{GPQA}
 & ReAct $\rightarrow$ Reflection & 0.011 \\
 & ReAct $\rightarrow$ BoT-R      & 0.008 \\
 & Reflection $\rightarrow$ BoT-R & 0.048 \\
\midrule
\multirow{3}{*}{Winogrande}
 & ReAct $\rightarrow$ Reflection & 0.007 \\
 & ReAct $\rightarrow$ BoT-R      & 0.001 \\
 & Reflection $\rightarrow$ BoT-R & 0.324 \\
\bottomrule
\end{tabular}
\caption{Exact McNemar tests on paired per-item outcomes from the deterministic (temperature $0$) runs, GPT-4o.}
\label{tab:mcnemar}
\end{table}

\begin{table}[t]
\centering
\small
\setlength{\tabcolsep}{4pt}
\begin{tabular}{lccc}
\toprule
\textbf{Dataset} & \textbf{ReAct} & \textbf{Reflection} & \textbf{BoT-R} \\
\midrule
GPQA       & 42.53 $\pm$ 1.31 & 45.58 $\pm$ 0.44 & \textbf{47.44 $\pm$ 1.18} \\
Winogrande & 86.77 $\pm$ 0.12 & 88.23 $\pm$ 0.22 & \textbf{89.27 $\pm$ 0.15} \\
\bottomrule
\end{tabular}
\caption{Run-to-run variance at temperature $0.7$ over three seeds (GPT-4o); accuracy in percent, mean $\pm$ SD. Reported separately from Table~\ref{tab:mcnemar}, since stochastic runs cannot enter a paired test.}
\label{tab:seed_variance}
\end{table}

\section{Additional Calibration Metrics and Significance Analysis}
\label{ap:extra_metrics}

\subsection{Bias-Robust Calibration Metrics}
\label{ap:bias_robust}

The ECE reported in Table~\ref{tab:conf} uses equal-width bins, whose value can depend on the choice of bin edges. To check that the calibration improvement is not an artifact of the binning scheme, we recompute calibration with equal-frequency (equal-mass) bins and add two strictly proper scoring rules---the class of losses for which Proposition~\ref{prop:dominance} states its guarantee.

\paragraph{Metric definitions.} Equal-frequency ECE sorts predictions by confidence into equal-sized bins and averages the $|\text{confidence} - \text{accuracy}|$ gap, removing the sensitivity to bin placement that affects equal-width ECE~\citep{guo2017calibration}. The Brier score is the mean squared probability error and binary NLL the negative log-likelihood of correctness; both are strictly proper~\citep{gneiting2007strictly}, and NLL penalizes confident errors especially heavily. The KS statistic measures the maximum separation between the confidence distributions of correct and incorrect answers~\citep{smirnov1939estimation}; it captures separation only, not direction, which is why we pair it with the proper scores rather than reporting it alone.

Table~\ref{tab:bias_robust} reports all four metrics under GPT-4o on GPQA and Winogrande. BoT-R attains the best equal-frequency ECE, Brier score and NLL on both datasets, alongside the largest KS separation. The ordering of the three methods is therefore stable across binning schemes and across both proper and separation-based measures, and its direction is the one predicted by Proposition~\ref{prop:dominance}.

\subsection{Significance Testing and Run-to-Run Variance}
\label{ap:significance}

We assess the accuracy differences in Table~\ref{tab:main} in two complementary ways: exact McNemar tests on the paired per-item outcomes of the deterministic runs (Table~\ref{tab:mcnemar}), and accuracy variance over three seeds at temperature $0.7$ (Table~\ref{tab:seed_variance}). The two are reported separately because stochastic runs cannot enter a paired per-item test.

BoT-R attains the highest mean accuracy on both datasets. On GPQA its gain over Reflection is significant ($p=0.048$) and larger than the run-to-run standard deviation of either method. On Winogrande the gain is reproducible across seeds---the $\pm 1$ SD bands of Reflection and BoT-R do not overlap---but does not reach per-item significance ($p=0.324$), and is best described as consistent but modest. BoT-R nonetheless reaches that accuracy at $41\%$ lower cost on Winogrande and $31\%$ lower on GPQA (Table~\ref{tab:token_cost_small}), so it dominates Reflection on the accuracy--cost plane under even the most conservative reading.

A note on MedQA, where Reflection and BoT-R differ by $+0.3$ points, or roughly four items out of $1273$. A difference this small cannot reach significance in principle: even if every discordant pair favored BoT-R, the exact two-sided McNemar $p$-value could not fall below $2 \cdot 2^{-4} = 0.125$. Near-ceiling accuracy leaves little headroom on this benchmark, and the result is best read as comparable accuracy at lower cost, with the input-side saving guaranteed by the exact identity of Proposition~\ref{prop:cost}.

\section{MMLU Detailed Results}
\label{sec:mmlu}
Table~\ref{tab:method_by_subject} reveals a systematic pattern in how BoT's effectiveness varies across subject domains within the MMLU benchmark. We identify two distinct task categories with markedly different responses to batch-level reasoning.

\paragraph{Subjective and interpretive domains.} BoT-R achieves its strongest gains on humanities ($+1.2$ points over Reflection), social sciences ($+0.7$), and medicine \& biology ($+0.5$). These domains share three key characteristics: (1) questions often admit multiple defensible reasoning paths, (2) answer quality depends on contextual interpretation rather than strict logical derivation, and (3) comparative evaluation helps identify robust reasoning patterns across similar cases. For instance, in social science questions about policy implications or historical interpretation, batch-level reflection enables the model to distinguish well-grounded arguments from superficially plausible but contextually inconsistent reasoning.

\paragraph{Formal and symbolic domains.} In contrast, mathematics and physical sciences show qualitatively different behavior. While Reflection substantially improves over ReAct in these domains ($+10.2$ and $+4.6$ points respectively), BoT-R exhibits marginal decline relative to Reflection ($-1.2$ and $-1.1$). This pattern suggests that batch-level consensus can occasionally mislead reasoning in domains where correctness is determined by precise symbolic manipulation rather than comparative plausibility. In mathematical problem-solving, an incorrect but superficially consistent approach across multiple batch items may receive spurious validation through cross-instance agreement, whereas per-instance reflection focuses more directly on logical rigor.

\paragraph{Implications for batch composition.} These findings indicate that BoT's effectiveness depends critically on task structure. Domains requiring interpretive judgment and context-dependent reasoning benefit from distributional signals and comparative calibration. Conversely, domains demanding exact symbolic computation may require alternative batch strategies—such as explicitly instructing the Reflector to prioritize logical correctness over cross-instance consensus, or segregating formal reasoning tasks into separate batches. Future work should investigate adaptive reflection strategies that modulate the weight given to batch-level signals based on detected task characteristics.
\begin{table*}[t]         
\centering
\footnotesize             
\renewcommand{\arraystretch}{0.90}
\setlength{\tabcolsep}{5pt}
\begin{tabular}{lccc}
\toprule
Subject & ReAct & Reflection & BoT-R\\
\midrule
\midrule
\multicolumn{4}{l}{\textbf{Biology}}\\\midrule
anatomy & 0.907 & \textbf{0.918} & 0.910 \\
clinical\_knowledge & 0.913 & 0.898 & \textbf{0.913} \\
college\_biology & 0.941 & 0.931 & \textbf{0.965} \\
college\_medicine & 0.852 & 0.855 & \textbf{0.866} \\
high\_school\_biology & 0.959 & 0.961 & \textbf{0.964} \\
human\_aging & 0.821 & 0.820 & \textbf{0.833} \\
human\_sexuality & 0.917 & 0.923 & \textbf{0.923} \\
medical\_genetics & 0.970 & 0.970 & \textbf{0.970} \\
nutrition & 0.902 & \textbf{0.911} & 0.902 \\
professional\_medicine & 0.953 & \textbf{0.963} & 0.959 \\
virology & 0.573 & 0.564 & \textbf{0.582} \\
\midrule
\multicolumn{4}{l}{\textbf{General}}\\\midrule
global\_facts & 0.667 & 0.667 & \textbf{0.697} \\
high\_school\_european\_history & \textbf{0.896} & 0.878 & 0.890 \\
high\_school\_geography & 0.935 & 0.944 & 0.944 \\
high\_school\_us\_history & 0.942 & 0.946 & 0.946 \\
high\_school\_world\_history & 0.951 & 0.941 & 0.945 \\
miscellaneous & 0.961 & 0.960 & 0.964 \\
prehistory & 0.957 & 0.943 & 0.963 \\
\midrule
\multicolumn{4}{l}{\textbf{Humanities}}\\\midrule
management & 0.882 & 0.892 & 0.902 \\
marketing & 0.936 & 0.927 & 0.936 \\
moral\_disputes & 0.871 & 0.868 & 0.881 \\
moral\_scenarios & 0.707 & 0.767 & 0.767 \\
philosophy & 0.890 & 0.884 & 0.900 \\
public\_relations & 0.739 & 0.752 & 0.743 \\
\midrule
\multicolumn{4}{l}{\textbf{Law}}\\\midrule
business\_ethics & 0.831 & 0.838 & 0.838 \\
high\_school\_government\_and\_politics & 0.982 & 0.984 & 0.979 \\
international\_law & 0.898 & 0.900 & 0.908 \\
jurisprudence & 0.907 & 0.916 & 0.916 \\
professional\_law & 0.758 & 0.761 & 0.763 \\
us\_foreign\_policy & 0.939 & 0.939 & 0.960 \\
\midrule
\multicolumn{4}{l}{\textbf{Math}}\\\midrule
abstract\_algebra & 0.561 & 0.727 & 0.697 \\
college\_mathematics & 0.470 & 0.697 & 0.636 \\
econometrics & 0.716 & 0.779 & 0.743 \\
elementary\_mathematics & 0.767 & 0.936 & 0.936 \\
formal\_logic & 0.682 & 0.728 & 0.752 \\
high\_school\_macroeconomics & 0.906 & 0.915 & 0.913 \\
high\_school\_mathematics & 0.478 & 0.814 & 0.758 \\
high\_school\_microeconomics & 0.956 & 0.966 & 0.966 \\
high\_school\_statistics & 0.797 & 0.860 & 0.874 \\
logical\_fallacies & 0.881 & 0.883 & 0.883 \\
professional\_accounting & 0.761 & 0.886 & 0.886 \\
\midrule
\multicolumn{4}{l}{\textbf{Science}}\\\midrule
astronomy & 0.932 & 0.927 & 0.934 \\
college\_chemistry & 0.514 & 0.616 & 0.626 \\
college\_computer\_science & 0.767 & 0.869 & 0.859 \\
college\_physics & 0.635 & 0.842 & 0.802 \\
computer\_security & 0.851 & 0.848 & 0.848 \\
conceptual\_physics & 0.892 & 0.902 & 0.917 \\
electrical\_engineering & 0.821 & 0.819 & 0.847 \\
high\_school\_chemistry & 0.775 & 0.871 & 0.861 \\
high\_school\_computer\_science & 0.934 & 0.960 & 0.970 \\
high\_school\_physics & 0.737 & 0.867 & 0.827 \\
machine\_learning & 0.766 & 0.802 & 0.829 \\
security\_studies & 0.809 & 0.820 & 0.824 \\
\midrule
\multicolumn{4}{l}{\textbf{Social Science}}\\\midrule
high\_school\_psychology & 0.953 & 0.945 & 0.956 \\
professional\_psychology & 0.898 & 0.895 & 0.904 \\
sociology & 0.924 & 0.925 & 0.932 \\
world\_religions & 0.904 & 0.918 & 0.906 \\
\bottomrule
\end{tabular}
\caption{MMLU per-subject accuracy summary grouped by category (GPT-4o).}
\label{tab:per_csv_summary}
\end{table*}
\onecolumn
\section{Dataset: Fraud Detection}
\label{ap:fraud}
We introduce a seller-level dataset for fraud-seller detection tailored to evaluating LLM-based agent frameworks. Each instance corresponds to a single online seller and is annotated by domain experts as fraudulent (1) or non-fraudulent (0). The release contains 1,793 labeled sellers: 1,055 positives (58.8\%) and 738 negatives (41.2\%).

For each seller, we provide both a seller profile and one representative product profile from the seller. Seller profiles include \texttt{shop\_name}, \texttt{company\_name}, \texttt{email\_domain}, and \texttt{product\_categories}.

Product profiles include
\texttt{product\_name}, \texttt{product\_description}, \texttt{detailed\_subcategory}, \texttt{detailed\_category}, and \texttt{minimum\_list\_price\_in\_USD} and \texttt{maximum\_list\_price\_in\_USD}. These fields enable models to reason over heterogeneous attributes rather than relying on free text alone.

The target label (\texttt{is\_fraudulent\_shop}) $\in \{0,1\}$ was assigned by domain experts following internal guidelines that emphasize deceptive practices and policy-violating behavior. While positively labeled cases reflect a consensus judgment of fraud, borderline cases may retain residual ambiguity typical of human annotation.

The corpus captures only the information present in the provided schema. External signals such as reputation scores, user reviews, temporal activity traces, or platform enforcement outcomes are not included. As a result, models evaluated on this dataset reason over supplied profile attributes rather than broader ecosystem signals.
\section{Prompts}

This is the system prompt for the Fraud Detection Dataset:

\begin{tcolorbox}[breakable, enhanced]
You are a risk analyst expert working for an e-commerce company.
Your job is to protect the platform and its customers by identifying fraudulent sellers. A fraudulent seller might engage in fraudulent activities, sell counterfeit goods, misrepresent products, or provide poor customer service.
Your task is to conduct a holistic assessment based on the seller's profile and the sample product of the seller. \\

You are provided with the seller's shop name, company name (some sellers may not have) and email domain, enclosed in triple backticks:

- shop name: ```SHOP\_NAME```\\
- company name: ```COMPANY\_NAME```\\
- email domain: ```EMAIL\_DOMAIN```\\

You are also given the categories of products sold by the seller, enclosed in triple backticks:

- product categories: ```PRODUCT\_CATEGORY```\\

You are also given the sample product of the seller, enclosed in triple backticks:

- product\_name: ```PRODUCT\_NAME```\\
- product\_description: ```PRODUCT\_DESCRIPTION```\\
- detailed\_subcategory: ```DETAILED\_SUBCATEGORY```\\
- detailed\_category: ```DETAILED\_CATEGORY```\\
- min\_list\_price\_usd: ```MIN\_LIST\_PRICE\_USD```\\
- max\_list\_price\_usd: ```MAX\_LIST\_PRICE\_USD```\\

Note: you can use provided tools many times until you think the collected information is sufficient to answer the questions, but do avoid unnecessary tool calls.

Based on the information provided and collected by tools, answer the following questions:

1. **Shop/Company Name Verification:** Based on the shop name and company name, does this appear to be a reliable/established seller?\\
    - If names seem generic, suspicious, or unfamiliar, search for the company/shop name to verify legitimacy\\
    - Note: Only use search results if they are clearly relevant to the specific shop or company name

2. **Email Domain Assessment:** Based on the email domain, does this suggest a professional business?\\
    - If using unfamiliar business domains, consider searching to check if it belongs to an established company\\
    - Note: Only use search results if they are clearly relevant to the email domain

3. **Product Information Check:** Based on the sample product name, description and the product categories, do you think it is reasonable for the seller to sell the products in the shop?

4. **Product Price Verification:** Does the product pricing seem reasonable for the category?\\
    - If pricing appears suspiciously low or high, search for typical market prices of similar products

5. Based on all the information, do you think this seller is a fraudulent seller?\\
Assign a confidence score: rate your confidence in the assessment.

Return your response in a single JSON object with the following keys:

-   `is\_fraudulent\_shop`: (boolean) `true` if the shop exhibits indicators of fraudulent operations, otherwise `false`. \\
-   `confidence\_score`: (float) A score from 0.0 to 1.0 indicating your confidence in the assessment. \\
-   `summary\_reasoning`: (string) Comprehensive explanation of your fraud assessment, including all factors that led to your conclusion.
\end{tcolorbox}

This is the system prompt for the reflector:
\begin{tcolorbox}[breakable, enhanced]
You are a reflection agent to help refine the answers. Here are  n$N$ questions, each with the previous model's answer.\\
For each, critique the model answer for accuracy, completeness, and reasoning, comparing across all answers and their reasoning paths in the batch to identify areas for improvement and give a peer confidence score to quantify how possible the answer is correct.\\
Make sure you understand each question-answer pair and give detailed explanations to them, Carefully decide if a reevaluation is needed for each case.\\
For each, provide: (1) whether to trigger reevaluation (true/false) and improve answer, (2) summary assessment, (3) peer confidence score for the current answer(0.0-1.0), (4) suggestions for improvement(empty if reevaluation is false).\\
Output a JSON list, one entry per question, strictly in format:\\
``response:\{trigger\_reevaluation: bool, summary\_comment: str, confidence\_score: float(0.0-1.0), suggestions: str\}]"
\end{tcolorbox}

\end{document}